\documentclass[11pt]{article}

\usepackage{natbib,fullpage}
\usepackage{bm,amsmath,amsthm,amssymb,multicol,algorithmic,algorithm,enumitem}
\usepackage{color,wrapfig}
\usepackage[textwidth=1cm,textsize=footnotesize]{todonotes}
\usepackage[final]{neurips_2019}

\usepackage{lipsum}
\usepackage[colorlinks=true,
linkcolor=red,
urlcolor=blue,
citecolor=blue]{hyperref}

\usepackage{xargs}
\usepackage{stmaryrd}

\newtheorem{Lemma}{Lemma}
\newtheorem{Prop}{Proposition}
\newtheorem{Theorem}{Theorem}

\newtheorem{assumption}{H\!\!}

\newtheorem{assumptionB}{B\!}

\newtheorem*{Lemma*}{Lemma}
\newtheorem*{Theorem*}{Theorem}
 \makeatletter
\renewenvironment{proof}[1][\proofname]{%
   \par\pushQED{\qed}\normalfont%
   \topsep6\p@\@plus6\p@\relax
   \trivlist\item[\hskip\labelsep\bfseries#1]%
   \ignorespaces
}{%
   \popQED\endtrivlist\@endpefalse
}
\makeatother
\usepackage{shortcuts_OPT}

    \makeatletter
    \def\multilimits@{\bgroup
  \Let@
  \restore@math@cr
  \default@tag
 \baselineskip\fontdimen10 \scriptfont\tw@
 \advance\baselineskip\fontdimen12 \scriptfont\tw@
 \lineskip\thr@@\fontdimen8 \scriptfont\thr@@
 \lineskiplimit\lineskip
 \vbox\bgroup\ialign\bgroup\hfil$\m@th\scriptstyle{##}$\hfil\crcr}
    \def\Sb{_\multilimits@}
    \def\endSb{\crcr\egroup\egroup\egroup}
\makeatother

\newtheoremstyle{t}         
    {\baselineskip}{2\topsep}      
    {\rm}                   
    {0pt}{\bfseries}  
    {}                      
    { }                      
    {\thmname{#1}\thmnumber{#2}.}

\theoremstyle{t}

\makeatletter
\DeclareRobustCommand*\cal{\@fontswitch\relax\mathcal}
\makeatother

\newcommand{\mathbbm}[1]{\text{\usefont{U}{bbm}{m}{n}#1}}

\begin{document}
\title{On the Global Convergence of (Fast) Incremental Expectation Maximization Methods}
\author{
  Belhal Karimi \\
  CMAP, \'{E}cole Polytechnique \\
  Palaiseau, France \\
  \texttt{belhal.karimi@polytechnique.edu} \\
  \And
  Hoi-To Wai \\
  The Chinese University of Hong Kong\\
  Shatin, Hong Kong\\
  \texttt{htwai@se.cuhk.edu.hk}
  \And
  Eric Moulines\\
  CMAP, \'{E}cole Polytechnique \\
  Palaiseau, France \\
  \texttt{eric.moulines@polytechnique.edu} \\
  \And
  Marc Lavielle \\
  INRIA Saclay\\
  Palaiseau, France \\
  \texttt{marc.lavielle@inria.fr} \\
}
\date{\today}

\maketitle

\begin{abstract}
\noindent The EM algorithm is one of the most popular algorithm for inference in latent data models. The original formulation of the EM algorithm does not scale to large data set, because the whole data set is required at each iteration of the algorithm. To alleviate this problem, \citet{neal1998view}  have  proposed   an incremental version of the EM (\IEM) in  which at each iteration the conditional expectation of the latent data (E-step)  is updated only for a mini-batch of observations. Another approach has been proposed by \citet{cappe2009line} in which the E-step is replaced by a stochastic approximation step, closely related to stochastic gradient. In this paper, we analyze incremental and stochastic version of the EM algorithm as well as the variance reduced-version of \citep{chen2018stochastic} in a common unifying framework. We also introduce a new version incremental version, inspired by the SAGA algorithm by \citet{defazio2014saga}. We establish non-asymptotic convergence bounds for global convergence.
Numerical applications are presented in this article to illustrate our findings.
\end{abstract}


\section{Introduction}
Many problems in machine learning pertain to tackling an empirical risk minimization  of the form
\beq \label{eq:em_motivate}
\min_{ \param \in \Param }~ \overline{\calL} ( \param ) \eqdef \Pen (\param) + \calL ( \param )~~\text{with}~~\calL ( \param ) = \frac{1}{n} \sum_{i=1}^n \calL_i( \param) \eqdef  \frac{1}{n} \sum_{i=1}^n \big\{ - \log g( y_i ; \param ) \big\}\eqs,
\eeq
where $\{y_i\}_{i=1}^n$ are the observations, $\Param$ is a convex subset of $\rset^d$ for the parameters,  $\Pen : \Param \rightarrow \rset$ is a smooth convex regularization function   and for each $\param \in \Param$, $g(y;\param)$ is the (incomplete) likelihood of each individual  observation. 
The objective function $ \overline{\calL} ( \param )$ is possibly \emph{non-convex} and is assumed to be lower bounded $ \overline{\calL} ( \param ) > - \infty$ for all $\param \in \Param$.
In the latent variable model,  $g(y_i ; \param)$, is the marginal of the
complete data likelihood defined as $f(z_i,y_i; \param)$, i.e. $g(y_i; \param) = \int_{\Zset} f (z_i,y_i;\param) \mu(\rmd z_i)$, where $\{ z_i \}_{i=1}^n$ are the (unobserved)
latent variables. We consider the setting where the complete data likelihood  belongs to the curved exponential family, \ie
\beq \label{eq:exp}
f(z_i,y_i; \param) = h  (z_i,y_i) \exp \big( \pscal{S(z_i,y_i)}{\phi(\param)} - \psi(\param) \big)\eqs,
\eeq
where $\psi(\param)$, $h(z_i,y_i)$ are scalar functions, $\phi(\param) \in \rset^k$ is a vector function, and $S(z_i,y_i) \in \rset^k$ is the complete data sufficient statistics.
Latent variable models are widely used in machine learning and statistics; examples include mixture models for density estimation, clustering document, and topic modelling; see \citep{mclachlan2007algorithm} and the references therein.

The basic "batch" EM (bEM) method iteratively computes a sequence of estimates $\{ \param^k, k \in \nset\}$ with an initial parameter $\param^0$. Each iteration of bEM is composed of two steps. In the {\sf E-step}, a surrogate function is computed as $\param \mapsto Q(\param,\param^{k-1}) = \sum_{i=1}^n Q_i(\param,\param^{k-1})$ where $Q_i(\param,\param') \eqdef - \int_{\Zset} \log f(z_i,y_i;\theta) p(z_i|y_i;\param') \mu(\rmd z_i)$ such that $p(z_i|y_i;\param) \eqdef f(z_i,y_i;\param)/ g(y_i,\param)$ is the conditional probability density of the latent variables $z_i$ given the observations $y_i$. When $f(z_i,y_i;\param)$ follows the curved exponential family model, the {\sf E-step} amounts to computing the conditional expectation of the complete data sufficient statistics, 
\begin{equation}
\label{eq:definition-overline-bss}
\overline{\bss}(\param)= \frac{1}{n} \sum_{i=1}^n \overline{\bss}_i(\param) \quad  \text{where}  \quad \overline{\bss}_i(\param)= \int_{\Zset} S(z_i,y_i) p(z_i|y_i;\param) \mu(\rmd z_i) \,.
\end{equation}
In the {\sf M-step}, the surrogate function is minimized producing a new fit of the parameter $\param^{k} = \argmax_{\param \in \Param} Q(\param,\param^{k-1})$.
The EM method has several appealing features -- it is monotone where the likelihood do not decrease at each iteration, invariant with respect to the parameterization,  numerically stable when the optimization set is well defined, etc.
The EM method  has been the subject of considerable interest since its formalization in \citep{dempster1977Maximum}.

With the sheer size of data sets today, the bEM method is not applicable as the {\sf E-step} \eqref{eq:definition-overline-bss} involves a full pass over the dataset of $n$ observations.
Several approaches based on stochastic optimization have been proposed to address this problem. 
\citet{neal1998view} proposed (but not analyzed) an incremental version of EM, referred to as the \IEM\ method. 
\citet{cappe2009line} developed the online EM (\SEM) method which uses a stochastic approximation procedure to track the sufficient statistics defined in \eqref{eq:definition-overline-bss}.
Recently, \citet{chen2018stochastic} proposed a variance reduced sEM (\SEMVR) method which is inspired by the SVRG algorithm popular in stochastic convex optimization \citep{johnson:zhang:2013}.
The applications of the above stochastic EM methods are numerous, especially with the \IEM\ and \SEM\ methods; e.g., \citep{ThiessonAccelerating2001}  for inference with missing data, \citep{ngChoice2003} for mixture models and unsupervised clustering, \citep{hinton2006fast} for inference of deep belief networks, \citep{hofmann2017probabilistic} for probabilistic latent semantic analysis, \citep{wainwright2008graphical,BleiVariational2017} for variational inference of graphical models and \citep{ablin2018algorithms} for Independent Component Analysis. 


This paper focuses on the theoretical aspect of stochastic EM methods
by establishing novel \emph{non-asymptotic} and \emph{global} convergence rates for them.  
Our contributions are as follows. 
\begin{itemize}
\item We offer two complementary views for the global convergence of EM methods -- one focuses on the parameter space, and one on the sufficient statistics space. On one hand, the EM method can be studied as an \emph{majorization-minimization} (MM) method in the parameter space.
On the other hand, the EM method can be studied as a \emph{scaled-gradient method} in the sufficient statistics space.
\item Based on the two views described, we derive non-asymptotic convergence rate for stochastic EM methods. First, we show that the \IEM\ method \citep{neal1998view} is a special instance of the MISO framework \citep{mairal2015incremental}, and takes ${\cal O}( n / \epsilon )$ iterations to find an $\epsilon$-stationary point to the ML estimation problem. Second, the \SEMVR\ method \citep{chen2018stochastic} is an instance of variance reduced stochastic scaled-gradient method, which takes ${\cal O}( n^{2/3} / \epsilon )$ iterations to find to an $\epsilon$-stationary point.
\item Lastly, we develop a Fast Incremental EM (\FIEM) method based on the SAGA algorithm \citep{defazio2014saga,reddi2016fast} for stochastic optimization. We show that the new method is again a scaled-gradient method with the same iteration complexity as \SEMVR. This new method offers trade-off between storage cost and computation complexity.
\end{itemize}
Importantly, our results capitalizes on the efficiency of stochastic EM methods applied on large datasets, and we support the above findings using numerical experiments. 

\paragraph{Prior Work} 
Since the empirical risk minimization problem \eqref{eq:em_motivate} is typically \emph{non-convex}, most prior work studying the convergence of EM methods considered either the \emph{asymptotic} and/or \emph{local} behaviors.
For the classical study, the global convergence to a stationary point (either a local minimum or a saddle point) of the bEM method  has been established by
\cite{wu1983convergence} (by making the arguments developed in \cite{dempster1977Maximum} rigorous). The global convergence is a direct consequence of the EM method to be monotone. It is also known that in the neighborhood of a stationary point and under regularity conditions, the local rate of convergence of the bEM is linear and is given by the amount of \emph{missing information} \cite[Chapters~3 and 4]{mclachlan2007algorithm}.

The convergence of the \IEM\ method was first tackled by
\citet{gunawardana2005convergence} exploiting the interpretation of the method as an alternating minimization procedure under the information geometric framework developed in \citep{csiszar:tusnady:1984}. Although the EM algorithm is presented as an alternation between the {\sf E-step} and {\sf M-step}, it is also possible to take a variational perspective on EM to view both steps as maximization steps.
Nevertheless, \citet{gunawardana2005convergence} assume that the latent variables take only a finite number of values and the order in which the observations are processed remains the same from one pass to the other.

More recently, the \emph{local but non-asymptotic convergence} of EM methods has been studied in several works.
These results typically require the initializations to be within a neighborhood of an isolated stationary point and the (negated) log-likelihood function to be strongly convex locally. Such conditions are either difficult to verify in general or have been derived only for specific models; see for example \citep{wang:gu:liu:2015,xu:hsu:maleki:2016,balakrishnan2017statistical} and the references therein.
The local convergence of \SEMVR\ method has been studied in \citep[Theorem~1]{chen2018stochastic} but under a pathwise global stability condition. The authors'  work \citep{karimi2019non} provided the first global non-asymptotic analysis of the online (stochastic) EM method \citep{cappe2009line}. In comparison, the present work analyzes the variance reduced variants of EM method. Lastly, it is worthwhile to mention that \citet{zhu2017high} analyzed a variance reduced \emph{gradient} EM method similar to \citep{balakrishnan2017statistical}.



\section{Stochastic Optimization Techniques for EM methods} \label{sec:sEM} 
Let $k \geq 0$ be the iteration number. The $k$th iteration of a generic stochastic EM method is composed of two sub-steps --- firstly,
\beq \label{eq:sestep}
\textsf{sE-step}:~\hat{\bss}^{(k+1)} = \hat{\bss}^{(k)} - \gamma_{k+1} \big( \hat{\bss}^{(k)} - \StocEstep^{(k+1)}  \big),
\eeq
which is a stochastic version of the {\sf E-step} in \eqref{eq:definition-overline-bss}. Note $\{ \gamma_{k} \}_{k=1}^\infty \in [0,1]$ is a sequence of step sizes, $\StocEstep^{(k+1)}$ is a proxy for $\overline{\bss}( \hat{\param}^{(k)} )$, and $\overline{\bss}$ is defined in \eqref{eq:definition-overline-bss}. Secondly, the {\sf M-step} is given by
\beq \label{eq:mstep}
\textsf{M-step:}~~\hat{\param}^{(k+1)} = \overline{\param}( \hat{\bm s}^{(k+1)} ) \eqdef \argmin_{ \param \in \Param } ~\big\{ \Pen( \param ) + \psi( \param) - \pscal{ \hat{\bm s}^{(k+1)} }{ \phi ( \param) } \big\},
\eeq
which depends on the sufficient statistics in the {\sf sE-step}. 
The stochastic EM methods differ in the way that $\StocEstep^{(k+1)}$ is computed. Existing methods employ stochastic approximation or variance reduction without the need to fully compute $\overline{\bss}( \hat{\param}^{(k)} )$.
To simplify notations, we define
\beq \label{eq:estep_upd}
\overline{\bss}_i^{(k)} \eqdef \overline{\bss}_i ( \hat{\param}^{(k)} )  = \int_{\Zset} S(z_{i},y_i) p(z_i|y_i;\hat{\param}^{(k)}) \mu(\rmd z_i)  \quad \text{and} \quad
\overline{\bss}^{(\ell)} \eqdef \overline{\bss}( \hat{\param}^{(\ell)} ) = \frac{1}{n} \sum_{i=1}^n \overline{\bss}_i^{(\ell)}.
\eeq
If $\StocEstep^{(k+1)} = \overline{\bss}^{(k)}$ and $\gamma_{k+1} = 1$,   \eqref{eq:sestep} reduces to the  {\sf E-step} in the classical bEM method.
To formally describe the stochastic EM methods, we let $i_k \in \inter$ be a random index drawn at iteration $k$ and $\tau_i^k = \max \{ k' : i_{k'} = i,~k' < k \}$ be the iteration index such that $i \in \inter$ is last drawn prior to iteration $k$. The proxy $\StocEstep^{(k+1)}$ in \eqref{eq:sestep} is drawn as:\vspace{-.2cm}
\begin{align}
&\emph{(\IEM\ \citep{neal1998view})} & \StocEstep^{(k+1)} &= \StocEstep^{(k)} + {\textstyle \frac{1}{n}}\big( \overline{\bss}_{i_k}^{(k)}  - \overline{\bss}_{i_k}^{(\tau_{i_k}^k)} \big) \label{eq:iem} \\
&\emph{(\SEM\ \citep{cappe2009line})} & \StocEstep^{(k+1)} &= \overline{\bss}_{i_k}^{(k)}  \label{eq:oem} \\
&\emph{(\SEMVR\ \citep{chen2018stochastic})} &\StocEstep^{(k+1)} &= \overline{\bss}^{(\ell(k))} +  \big( \overline{\bss}_{i_k}^{(k)}  - \overline{\bss}_{i_k}^{(\ell(k))}   \big) \label{eq:svrgem}
\end{align}
The stepsize is set to $\gamma_{k+1} = 1$ for the \IEM\ method; $\gamma_{k+1} = \gamma$ is  constant for the \SEMVR\ method.
In the original version of the \SEM\ method, the sequence of step
 $\gamma_{k+1}$ is a diminishing step size. Moreover, for \IEM\ we initialize with $\StocEstep^{(0)} = \overline{\bss}^{(0)}$; for \SEMVR, we set an epoch size of $m$ and define $\ell(k) \eqdef m \lfloor k/m \rfloor$ as the first iteration number in the epoch that iteration $k$ is in. \vspace{-.2cm}

\paragraph{\FIEM} Our analysis framework can handle a new, yet natural application of a popular variance reduction technique to the EM method. The new method, called \FIEM, is developed from the SAGA method \citep{defazio2014saga} in a similar vein as in  \SEMVR.

For iteration $k \geq 0$, the \FIEM\ method draws \emph{two} indices \emph{independently} and uniformly as $i_k, j_k \in \inter$. In addition to $\tau_i^k$ which was defined \wrt $i_k$, we define $t_j^k = \{ k' : j_{k'} = j , k' < k \}$ to be the iteration index where the sample $j \in \inter$ is last drawn as $j_k$ prior to iteration $k$. With the initialization $\overline{\StocEstep}^{(0)} = \overline{\bss}^{(0)}$, we use a slightly different update rule from SAGA inspired by \citep{reddi2016fast}, as described by the following recursive updates
\beq \label{eq:sagaem}
\StocEstep^{(k+1)} = \overline{\StocEstep}^{(k)} + \big( \overline{\bss}_{i_k}^{(k)}  - \overline{\bss}_{i_k}^{(t_{i_k}^k)} \big),~~
\overline{\StocEstep}^{(k+1)} = \overline{\StocEstep}^{(k)} + n^{-1}
\big( \overline{\bss}_{j_k}^{(k)}  - \overline{\bss}_{j_k}^{(t_{j_k}^k)} \big).
\eeq
\begin{wrapfigure}[18]{r}{.5\linewidth}\vspace{-0.2cm}
\begin{minipage}{\linewidth}
 \algsetup{indent=1em}
\begin{algorithm}[H]
\caption{Stochastic EM methods.}\label{alg:sem}
  \begin{algorithmic}[1]
  \STATE \textbf{Input:} initializations $\hat{\param}^{(0)} \leftarrow 0$, $\hat{\bss}^{(0)} \leftarrow \overline{\bss}^{(0)}$, $K_{\sf max}$ $\leftarrow$ max.~iteration number. \STATE Set the terminating iteration number, $K \in \{0,\dots,K_{\sf max}-1\}$, as a discrete r.v.~with:\vspace{-.1cm}
  \beq \label{eq:random}
   P( K = k ) = \frac{ \gamma_{k} }{\sum_{\ell=0}^{K_{\sf max}-1} \gamma_\ell}.\vspace{-.2cm}
  \eeq
  \FOR {$k=0,1,2,\dots, K$}
  \STATE Draw index $i_k \in \inter$ uniformly (and $j_k \in \inter$ for \FIEM).
   \STATE Compute the surrogate sufficient statistics $\StocEstep^{(k+1)}$ using \eqref{eq:oem} or \eqref{eq:iem} or \eqref{eq:svrgem} or \eqref{eq:sagaem}.
   \STATE Compute $\hat{\bss}^{(k+1)}$ via the {\sf sE-step} \eqref{eq:sestep}.
   \STATE Compute $\hat{\param}^{(k+1)}$ via the {\sf M-step} \eqref{eq:mstep}.
\ENDFOR
\STATE \textbf{Return}: $\hat{\param}^{(K)}$.
  \end{algorithmic}
\end{algorithm}\vspace{.1cm}
\end{minipage}\end{wrapfigure}
where we set a constant step size as $\gamma_{k+1} = \gamma$.

In the above, the update of $\StocEstep^{(k+1)}$ corresponds to an \emph{unbiased estimate} of $\overline{\bss}^{(k)}$, while the update for $\overline{\StocEstep}^{(k+1)}$ maintains the structure that $\overline{\StocEstep}^{(k)} = n^{-1} \sum_{i=1}^n \overline{\bss}_i^{(t_i^k)}$ for any $k \geq 0$.
The two updates of \eqref{eq:sagaem} are based on two different and independent indices $i_k,j_k$ that are randomly drawn from $\inter[n]$. This is used for our fast convergence analysis in Section~\ref{sec:main}.

We summarize the \IEM, \SEMVR, \SEM, \FIEM\ methods in Algorithm~\ref{alg:sem}. The random termination number \eqref{eq:random} is inspired by \citep{ghadimi2013stochastic} which enables one to show non-asymptotic convergence to stationary point for non-convex optimization. Due to their stochastic nature, the per-iteration complexity for all the stochastic EM methods are independent of $n$, unlike the bEM method. They are thus applicable to large datasets with $n \gg 1$.

\subsection{Example: Gaussian Mixture Model} \label{sec:gmm_main}
We discuss an example of learning a Gaussian Mixture Model (GMM) from a set of $n$ observations $\{ y_i \}_{i=1}^n$. We focus on a simplified setting where there are $M$ components of unit variance and unknown means, the GMM is parameterized by $\param = ( \{ \omega_m \}_{m=1}^{M-1} , \{ \mu_m \}_{m=1}^M ) \in \Theta = \Delta^M \times \rset^M$, where $\Delta^M \subseteq \rset^{M-1}$ is the reduced $M$-dimensional probability simplex [see \eqref{eq:const0_main}]. We use the penalization  
$\Pen(\param)= \frac{\delta}{2}\sum_{m=1}^M \mu_m^2 - \log \Dir(\bomega; M, \epsilon)$ where $\delta > 0$ and $\Dir(\cdot; M,\epsilon)$ is the $M$ dimensional symmetric Dirichlet distribution with concentration parameter $\epsilon > 0$.
Furthermore, 
we use $z_i \in \inter[M]$ as the latent label. 
The complete data log-likelihood is given by
\beq \label{eq:comp_like}  
\log f( z_i, y_i; \param) = 
\sum_{m=1}^{M} \indiacc{m=z_i} \left[ \log(\omega_m) - \mu_m^2/2 \right] + \sum_{m=1}^M \indiacc{m=z_i} \mu_m y_i + {\rm constant} ,
\eeq
where $\indiacc{m=z_i} = 1$ if $m=z_i$; otherwise $\indiacc{m=z_i} = 0$. 
The above can be rewritten in the same form as \eqref{eq:exp}, particularly with $S( y_i,z_i ) \equiv ( s_{i,1}^{(1)}, ..., s_{i,M-1}^{(1)} , s_{i,1}^{(2)}, ... ,s_{i,M-1}^{(2)}, s_i^{(3)} )$ and $\phi( \param ) \equiv ( \phi_1^{(1)}(\param), ..., \phi_{M-1}^{(1)}(\param), \phi_1^{(2)}(\param), ..., \phi_{M-1}^{(2)}(\param), \phi^{(3)}(\param) )$ such that
\beq \label{eq:gmm_exp0}
\begin{split}
& s_{i,m}^{(1)} = \indiacc{z_i = m}, \quad \phi_m^{(1)}(\param) = \{\log(\omega_m) -{\mu_m^2} / {2} \} - \{\log(1 - {\textstyle  \sum_{j=1}^{M-1}} \omega_j) -  {\mu_M^2} / {2} \},\\
& s_{i,m}^{(2)} =   \indiacc{z_i = m} y_i, \quad \phi^{(2)}_m(\param) =  {\mu_m}, \quad s_i^{(3)} = y_i, \quad \phi^{(3)}(\param) = \mu_M,
\end{split}
\eeq
and $\psi(\param) =   - \{\log(1 - \sum_{m=1}^{M-1} \omega_m) - {\mu_M^2} / {2 \sigma^2} \}$.
To evaluate the {\sf sE-step}, the conditional expectation required by \eqref{eq:estep_upd} can be computed in closed form, as they depend on $\EE_{\hat{\param}^{(k)}}[ \mathbbm{1}_{\{z_i=m\}} | y= y_{i} ]$ and $\EE_{\hat{\param}^{(k)}}[ y_i  \mathbbm{1}_{\{z_i=m\}} | y= y_{i} ]$.
Moreover, the {\sf M-step} \eqref{eq:mstep} solves a strongly convex problem and can be computed in closed form. Given a sufficient statistics ${\bm s} \equiv ( {\bm s}^{(1)}, {\bm s}^{(2)}, s^{(3)} )$, the solution to \eqref{eq:mstep} is:
\beq \label{eq:mstep_gmm0}
\overline{\param} ( {\bm s} )
= \left(
\begin{array}{c}
( 1+\epsilon M )^{-1} \big( {s}_1^{(1)} + \epsilon, \dots,  {s}_{M-1}^{(1)} + \epsilon \big)^\top \vspace{.2cm}\\
 \big( ({s}_1^{(1)} + \delta )^{-1} {s}_1^{(2)}  , \dots, ({s}_{M-1}^{(1)} + \delta )^{-1} {s}_{M-1}^{(2)}  \big)^\top \vspace{.2cm} \\
  \big(1 - \sum_{m=1}^{M-1}s_m^{(1)} +  \delta\big)^{-1} \big( s^{(3)} - \sum_{m=1}^{M-1} s_m^{(2)} \big)
\end{array}
\right).
\eeq
The next section presents the main results of this paper for the convergence of stochastic EM methods. We shall use the above example on GMM to illustrate the required assumptions.



\section{Global Convergence of Stochastic EM Methods} \label{sec:main}
We establish non-asymptotic rates for the \emph{global convergence} of the stochastic EM methods. We show that the \IEM\ method is an instance of the incremental MM method; while \SEMVR, \FIEM\ methods are instances of variance reduced \emph{stochastic} \emph{scaled gradient} methods. As we will see, the latter interpretation allows us to establish fast convergence rates of \SEMVR\ and \FIEM\ methods.
Detailed proofs for the theoretical results in this section are relegated to the appendix.

First, we list a few assumptions which will enable the convergence analysis performed later in this section. Define:
\beq \textstyle\label{eq:sset}
\Sset \eqdef  
\set{ \sum_{i=1}^n \alpha_i \bss_i }{ \bss_i \in {\rm conv}\set{S(z,y_i)}{z \in \Zset},~\alpha_i \in [-1,1],~i \in \inter },
\eeq
where ${\rm conv} \{ A \}$ denotes the closed convex hull of the set $A$. From \eqref{eq:sset}, we observe that the \IEM, \SEMVR, and \FIEM\ methods generate $\hat{\bm s}^{(k)} \in \Sset$ for any $k \geq 0$.  Consider:
\begin{assumption}\label{ass:compact}
The sets $\Zset, \Sset$ are compact. There exists constants $C_{\Sset}, C_{\Zset}$ such that:
\beq \textstyle \label{eq:compact}
C_{\Sset} \eqdef \max_{ \bss, \bss' \in \Sset } \| \bss - \bss' \| < \infty,~~~~C_{\Zset} \eqdef \max_{i \in \inter} \int_{\Zset} | S(z,y_i) | \mu( \rmd z ) < \infty.
\eeq
\end{assumption}
H\ref{ass:compact} depends on the latent data model used and can be satisfied by several practical models. For instance, the GMM in Section~\ref{sec:gmm_main} satisfies \eqref{eq:compact} as the sufficient statistics are composed of indicator functions and observations. Other examples can also be found in Section~\ref{sec:num}. 
Denote by $\jacob{\kappa}{\param}{\param'}$ the Jacobian of the function $\kappa: \param \mapsto \kappa(\param)$ at $\param' \in \Param$.
Consider:
\begin{assumption}\label{ass:regularity-phi-psi}
The function $\phi$ is smooth and bounded on ${\rm int}(\Param)$, \ie the interior of $\Param$. For all $\param,\param' \in {\rm int}(\Param)^2$,
$\|\jacob{\phi}{\param}{\param} - \jacob{\phi}{\param}{\param'} \| \leq \Lip{\phi} \| \param - \param' \|$ and $\| \jacob{\phi}{\param}{\param'} \| \leq C_{\phi}$.
\end{assumption}
\begin{assumption}\label{ass:expected}
The conditional distribution is smooth on ${\rm int}(\Param)$. For any $i \in \inter$, $z \in \Zset$, $\param, \param' \in {\rm int} (\Param)^2$, we have
$\big| p( z | y_i; \param ) - p( z | y_i; \param' ) \big| \leq  \Lip{p} \| \param - \param' \|$.
\end{assumption}
\begin{assumption} \label{ass:reg}
For any $\bm{s} \in \Sset$, the function $\param \mapsto L(s,\param) \eqdef \Pen( \param ) + \psi( \param) - \pscal{ \bss}{ \phi ( \param) }$ admits a unique global minimum $\mstep{\bss} \in {\rm int}(\Param)$.
In addition, $\jacob{\phi}{\param}{\overline{\param}(\bss )}$ is full rank and $\overline{\param}( \bss )$ is $\Lip{\theta}$-Lipschitz.
\end{assumption}
Under H\ref{ass:compact}, the assumptions
H\ref{ass:regularity-phi-psi} and H\ref{ass:expected} are standard for the curved exponential family distribution and the conditional probability distributions, respectively; H\ref{ass:reg} can be enforced by designing a strongly convex regularization function $\Pen( \param )$ tailor made for $\Param$. For instance, the penalization for GMM in Section~\ref{sec:gmm_main} ensures $\param^{(k)}$ is unique and lies in ${\rm int}( \Delta^M ) \times \rset^M$, which can further imply the second statement in H\ref{ass:reg}.
We remark that for H\ref{ass:expected}, it is possible to define the Lipschitz constant $\Lip{p}$ independently for each data $y_i$ to yield a refined characterization. We did not pursue such assumption to keep the notations simple.

Denote by $\hess{L}{\param}(\bss,\param)$ the Hessian w.r.t to $\param$ for a given value of $\bss$ of the function $\param \mapsto L(\bss,\param)= \Pen(\param) + \psi(\param) -\pscal{\bss}{\phi(\param)}$, and define
\beq\label{eq:Bss}
\operatorname{B}( \bss ) \eqdef\jacob{ \phi }{ \param }{ \mstep{\bss} } \Big( \hess{L}{\param}( {\bss},  \mstep{\bss} )  \Big)^{-1} \jacob{ \phi }{ \param }{ \mstep{\bss} }^\top.
\eeq
\begin{assumption}\label{ass:eigen}
It holds that $ \upsilon_{\max} \eqdef \sup_{\bss \in \Sset} \| \operatorname{B}( \bss ) \| < \infty$ and $0 < \upsilon_{\min}  \eqdef \inf_{\bss \in \Sset} \lambda_{\rm min} ( \operatorname{B}( \bss ) )$.
There exists a constant $\Lip{B}$ such that for all $\bss, \bss' \in \Sset^2$, we have $ \| \operatorname{B}( \bss ) - \operatorname{B}( \bss' )  \| \leq \Lip{B} \| {\bss} - {\bss}' \|$.
\end{assumption}
Again, H\ref{ass:eigen} is satisfied by practical models. For GMM in Section~\ref{sec:gmm_main}, it can be verified by deriving the closed form expression for $\operatorname{B}( \bss )$ and using H\ref{ass:compact}; also see the other example in Section~\ref{sec:num}.  The derivation is, however, technical and will be relegated to the supplementary material. 

Under H\ref{ass:compact}, we have $\| \hat{\bm s}^{(k)} \| < \infty$ since $\Sset$ is compact. On the other hand, under H\ref{ass:reg}, the EM methods generate $\hat{\param}^{(k)} \in {\rm int}( \Param )$ for any $k \geq 0$. Overall, these assumptions ensure that the EM methods operate in a `nice' set throughout the optimization process.


\subsection{Incremental EM method}
We show that the \IEM\ method is a special case of the MISO method \citep{mairal2015incremental} utilizing the majorization minimization (MM) technique. The latter is a common technique for handling non-convex optimization.
We begin by defining a surrogate function that majorizes ${\cal L}_i$:
\begin{equation}
\label{eq:definition-surrogate}
\surrogate_i (\param ; \param' ) \eqdef
-\int_{\Zset} \left\{ \log f(z_i,y_i;\param) - \log p(z_i|y_i;\param') \right\} p(z_i|y_i;\param') \mu(\rmd z_i) \eqsp.
\end{equation}
The second term inside the bracket is a constant that does not depend on the first argument $\param$. Since $f(z_i,y_i;\param)= p(z_i|y_i;\param) g(y_i;\param)$, for all $\param' \in \Param$, we get $\surrogate_i(\param'; \param')= - \log g(y_i;\param')= \calL_i(\param')$. For all $\param,\param' \in \Param$, applying the Jensen inequality shows
\beq
\surrogate_i(\param,\param') - \calL_i(\param)= \int \log \frac{p(z_i|y_i;\param')}{p(z_i|y_i;\param)} p(z_i|y_i;\param') \mu(\rmd z_i) \geq 0
\eeq
which is the Kullback-Leibler divergence between the conditional distribution of the latent data $p(\cdot|y_i;\param)$ and $p(\cdot|y_i;\param')$.
Hence, for all $i \in \inter$, $\surrogate_i (\param ; {\param}' )$ is a majorizing surrogate to $\calL_i( \param)$, \ie it satisfies for all $\param,\param' \in \Param$,
$\surrogate_i (\param ; \param' )  \geq {\cal L}_i( \param)$  with equality  when $\param = \param'$. For the special case of curved exponential family distribution, the {\sf M-step} of the \IEM\ method is expressed as
\beq \label{eq:mstep_iem}
\begin{split}
\textstyle
\hat{\param}^{(k+1)} & \textstyle \in
 \argmin_{ \param \in \Param } ~ \big\{ \Pen( \param )+ n^{-1} \sum_{i=1}^n \surrogate_i ( \param ; \hat{\param}^{(\tau_i^{(k+1)})} ) \big\}\\
 & \textstyle =  \argmin_{ \param \in \Param } ~ \Big\{ \Pen( \param )+ \psi(\param) - \pscal{n^{-1} \sum_{i=1}^n \overline{\bss}_i^{( \tau_i^{k+1} )}}{\phi(\param)} ) \Big\}.
 \end{split}
\eeq
The \IEM\ method can be interpreted through the MM technique --- in the {\sf M-step},
$\hat\param^{(k+1)}$ minimizes an upper bound of $\overline\calL( \param)$, while the {\sf sE-step} updates the surrogate function in \eqref{eq:mstep_iem} which tightens the upper bound.
Importantly, the error between the surrogate function and ${\cal L}_i$ is a smooth function:
\begin{Lemma}\label{lem:lips_theta}
Assume H\ref{ass:compact}, H\ref{ass:regularity-phi-psi}, H\ref{ass:expected}, H\ref{ass:reg}. Let $e_i ( \param; \param' ) \eqdef Q_i( \param; \param' ) - {\cal L}_i( \param )$. For any $\param, \bar{\param}, \param' \in \Param^3$, we have
$\| \grd e_i( \param; \param' ) - \grd e_i( \bar{\param} ; \param' ) \| \leq \Lip{e} \| \param - \bar\param \| $, where $\Lip{e} \eqdef C_{\phi} C_{\Zset} \Lip{p} +  C_{\Sset} \Lip{\phi}$.
\end{Lemma}

For \emph{non-convex} optimization such as \eqref{eq:em_motivate}, it has been shown \citep[Proposition~3.1]{mairal2015incremental} that the incremental MM method converges asymptotically to a stationary solution of a problem.
We strengthen their result by establishing a non-asymptotic rate, which is new to the literature.
\begin{Theorem} \label{thm:iem}
Consider the \IEM\ algorithm, \ie Algorithm~\ref{alg:sem} with \eqref{eq:iem}.  Assume H\ref{ass:compact}, H\ref{ass:regularity-phi-psi}, H\ref{ass:expected}, H\ref{ass:reg}.
For any $K_{\max} \geq 1$, it holds that
\beq \label{eq:iem_bdd}
\EE[ \| \grd \overline\calL( \hp{K} ) \|^2 ] \leq n \!~  \frac{2 \Lip{e}}{K_{\sf max}}  \!~ \EE \big[ \overline\calL ( \hp{0} ) - \overline\calL ( \hp{K_{\sf max}} ) \big],
\eeq
where $\Lip{e}$ is defined in Lemma~\ref{lem:lips_theta} and $K$ is a uniform random variable on  $\inter[{0,K_{\max}-1}]$ [cf.~\eqref{eq:random}] independent of the $\{i_k\}_{k=0}^{K_{\max}}$.
\end{Theorem}
We remark that under suitable assumptions, our analysis in Theorem~\ref{thm:iem}  also extends to several non-exponential family distribution models.

\subsection{Stochastic EM as Scaled Gradient Methods}
We interpret the \SEMVR\ and \FIEM methods as \emph{scaled gradient} methods on the sufficient statistics $\hat{\bss}$, tackling a \emph{non-convex} optimization problem.
The benefit of doing so is that we are able to demonstrate a faster convergence rate for these methods through motivating them as \emph{variance reduced} optimization methods. The latter is shown to be more effective when handling large datasets \citep{reddi2016fast,reddi2016stochastic,allen2016variance} than traditional stochastic/deterministic optimization methods. 
To set our stage, we consider the minimization problem:
\beq\label{eq:em_sspace}
\min_{ {\bss} \in \Sset }~  V ( {\bss} ) \eqdef \overline\calL( \op(\bss) ) = 
\Pen (  \op(\bss) ) + \frac{1}{n} \sum_{i=1}^n {\cal L}_i (  \op(\bss) )
,
\eeq
where $\overline{\param} ( {\bss})$ is the unique map defined in the {\sf M-step} \eqref{eq:mstep}. 
We first show that the stationary points of \eqref{eq:em_sspace} has a one-to-one correspondence with the stationary points of \eqref{eq:em_motivate}:
\begin{Lemma} \label{lem:global}
For any $\bss \in \Sset$, it holds that 
\beq \label{eq:grd_equiv}
\grd_{ \bss} V( {\bss} )  = \jacob{\overline{\param}}{\bss}{\bss}^\top \grd_\param \overline\calL( \overline\param( {\bss} ) ).
\eeq
Assume H\ref{ass:reg}. If $\bss^\star \in \set{ \bss \in \Sset }{\grd_{\bss} V(\bss )=0}$, then $\overline{\param}(\bss^\star) \in \set{\param \in \Param }{\grd_\param \overline{\calL}(\param)=0}$. Conversely, if $\param^* \in \set{\param \in \Param}{\grd_\param \overline{\calL}(\param)=0}$, then $\bss^*= \overline{\bss}(\param^*) \in \set{ \bss \in \Sset }{\grd_{\bss} V(\bss)=0}$.
\end{Lemma}
The next lemmas show that the update direction, $\hs{k} - \StocEstep^{(k+1)}$, in the {\sf sE-step} \eqref{eq:sestep} of \SEMVR\ and \FIEM\ methods is a \emph{scaled gradient} of $V(\bss)$. 
We first observe the following conditional expectation:
\beq \label{eq:expected_grad}
\EE \big[ \hs{k} - \StocEstep^{(k+1)}  | {\cal F}_k \big] = \hs{k} - \os^{(k)} 
= \hat{\bss}^{(k)} - \os( \op (\hs{k} ) ),
\eeq
where ${\cal F}_k$ is the $\sigma$-algebra generated by $\{ i_0, i_1,\dots, i_k \}$ (or $\{ i_0,j_0,\dots, i_k,j_k\}$ for \FIEM).

The difference vector ${\bss} - \overline{\bss}( \overline{\param} ({\bss}))$ and the gradient vector $\grd_{\bss} V ( {\bss} )$ are correlated, as we observe:
\begin{Lemma} \label{lem:semigrad}
Assume H\ref{ass:reg},H\ref{ass:eigen}. For all $\bss \in \Sset$,
\beq \label{eq:semigrad}
\upsilon_{\min}^{-1} \pscal{\grd V ( {\bss} ) }{ {\bss} - \os( \op ({\bss})) }
\geq \big\| {\bss} - \os( \op ({\bss})) \big\|^2 \geq \upsilon_{\max}^{-2} \| \grd V ( {\bss} ) \|^2,
\eeq
\end{Lemma}
Combined with \eqref{eq:expected_grad}, the above lemma shows that the update direction in the {\sf sE-step} \eqref{eq:sestep} of \SEMVR\ and \FIEM\ methods is a \emph{stochastic scaled gradient} where $\hat{\bss}^{(k)}$ is updated with a stochastic direction whose mean is correlated with $\grd V( \bss )$.

Furthermore, the expectation step's operator and the objective function in \eqref{eq:em_sspace} are smooth functions:
\begin{Lemma} \label{lem:smooth}
Assume H\ref{ass:compact}, H\ref{ass:expected}, H\ref{ass:reg}, H\ref{ass:eigen}.  
For all $\bss,\bss' \in \Sset$ and $i \in \inter$, we have
\beq \label{eq:smooth}
\| \overline{\bss}_i ( \overline{\param} ({\bss})) - \overline{\bss}_i ( \overline{\param} ({\bss}' )) \| \leq \Lip{{\bss}} \| {\bss} - {\bss}' \|,~~\| \grd  V ( {\bss} ) - \grd  V ( {\bss}' ) \| \leq \Lip{V} \| {\bss} - {\bss}' \|,
\eeq
where $\Lip{\bss} \eqdef C_{\Zset} \Lip{p} \Lip{\theta}$ and $\Lip{V}  \eqdef \upsilon_{\max} \big( 1 + \Lip{{\bss}} \big) + \Lip{B} C_{\Sset}$.
\end{Lemma}
The following theorem establishes the (fast) non-asymptotic convergence rates of \SEMVR\ and \FIEM\ methods, which are similar to \citep{reddi2016fast,reddi2016stochastic,allen2016variance}:
\begin{Theorem} \label{thm:svrg}
Assume H\ref{ass:compact}, H\ref{ass:expected}, H\ref{ass:reg}, H\ref{ass:eigen}. Denote $\overline{L}_{\sf v} = \max\{ \Lip{V}, \Lip{\bss} \}$ with the constants in Lemma~\ref{lem:smooth}.
\begin{itemize}[leftmargin=5.5mm]
\item Consider the \SEMVR\ method, \ie Algorithm~\ref{alg:sem} with  \eqref{eq:svrgem}. There exists a universal constant $\mu \in (0,1)$ (independent of $n$) such that if we set the step size as $\gamma = \frac{\mu \upsilon_{\min} }{ \overline{L}_{\sf v} n^{2/3}}$ and the epoch length as $m = \frac{n }{2 \mu^2 \upsilon_{\min}^2 +\mu }$, then for any $K_{\sf max} \geq 1$ that is a multiple of $m$, it holds that
\beq \label{eq:svrgem_bdd}
\EE[ \| \grd V( \hs{K} ) \|^2 ] \leq  n^{\frac{2}{3}} \!~ \frac{2 \overline{L}_{\sf v} }{\mu K_{\sf max}} \frac{ \upsilon_{\max}^2 }{ \upsilon_{\min}^2 } \!~ \EE[ V( \hs{0} ) - V( \hs{K_{\sf max}}) ] .
\eeq
\item Consider the \FIEM\ method, \ie Algorithm~\ref{alg:sem} with  \eqref{eq:sagaem}. Set $\gamma = \frac{\upsilon_{\min}}{\alpha \overline{L}_{\sf v} n^{2/3}}$ such that $\alpha = \max\{ 6, 1 + 4 \upsilon_{\min} \}$. For any $K_{\sf max} \geq 1$, it holds that
\beq \label{eq:sagaem_bdd}
\EE[ \| \grd V( \hs{K} ) \|^2 ] \leq n^{\frac{2}{3}} \!~ \frac{ \alpha^2 \overline{L}_{\sf v} }{K_{\sf max}}  \frac{ \upsilon_{\max}^2 }{ \upsilon_{\min}^2 }\EE \big[ V( \hs{0} ) - V( \hs{K_{\sf max}} ) \big] .
\eeq
\end{itemize}
We recall that $K$ in the above is a uniform and independent r.v.~chosen from $\inter[K_{\sf max}-1]$ [cf.~\eqref{eq:random}].
\end{Theorem}
In the supplementary materials, we also provide a local convergence analysis for \FIEM\ method which shows that the latter can achieve linear rate of convergence \emph{locally} under a similar set of assumptions used in \citep{chen2018stochastic} for \SEMVR\ method.

\paragraph{Comparing \IEM, \SEMVR, and \FIEM}
Note that by \eqref{eq:grd_equiv} in Lemma~\ref{lem:global}, if $\| \grd_{\bss} V( \hat{\bss} ) \|^2 \leq \epsilon$, then $\| \grd_{\param} \overline{\cal L}( \op( \hat\bss ) ) \|^2 = {\cal O}( \epsilon )$, and vice versa, where the hidden constant is independent of $n$. In other words, the rates for \IEM, \SEMVR, \FIEM\ methods in Theorem~\ref{thm:iem} and \ref{thm:svrg} are comparable. 

Importantly, the theorems show an intriguing comparison -- to attain an $\epsilon$-stationary point with $\| \grd_{\param} \overline{\cal L}( \op( \hat\bss ) ) \|^2 \leq \epsilon $ or $\| \grd_{\bss} V( \hat{\bss} ) \|^2 \leq \epsilon$, the \IEM\ method requires ${\cal O}( n / \epsilon )$ iterations (in expectation) while the \SEMVR, \FIEM\ methods require only ${\cal O}( n^{\frac{2}{3}} / \epsilon )$ iterations (in expectation). This comparison can be surprising since the \IEM\ method is a monotone method as it guarantees decrease in the objective value; while the \SEMVR, \FIEM\ methods are non-monotone. Nevertheless, it aligns with the recent analysis on stochastic variance reduction methods on non-convex problems. In the next section, we confirm the theory by observing a similar behavior numerically.

%


\section{Numerical Examples}
\label{sec:num}
\subsection{Gaussian Mixture Models}
As described in Section~\ref{sec:gmm_main}, our goal is to fit a GMM model to a set of $n$ observations $\{y_i\}_{i=1}^n$ whose distribution is modeled as a Gaussian mixture of $M$ components, each with a unit variance. Let $z_i \in \inter[M]$ be the latent labels, the complete log-likelihood is given in \eqref{eq:comp_like},
where $\param \eqdef (\bomega, \bmu)$ with $\bomega= \{\omega_{m}\}_{m=1}^{M-1}$ are the mixing weights with the convention $\omega_M= 1 - \sum_{m=1}^{M-1} \omega_m$  and $\bmu= \{\mu_m \}_{m =1}^M$ are the means. 
The constraint set on $\param$ is given by
\beq \textstyle \label{eq:const0_main}
\Param = \{ \omega_m,~m=1,...,M-1 : \omega_m \geq 0,~\sum_{m=1}^{M-1} \omega_m \leq 1\} \times \{ \mu_m \in \rset ,~m=1,...,M \}.
\eeq
In the following experiments of synthetic data, we generate samples from a GMM model with $M=2$ components with two mixtures with means $\mu_1 = - \mu_2 = 0.5$, see Appendix~\ref{app:gmm} for details of the implementation and satisfaction of model assumptions for GMM inference.
We aim at verifying the theoretical results in Theorem~\ref{thm:iem} and \ref{thm:svrg} of the dependence on sample size $n$.

\paragraph{Fixed sample size} We use $n = 10^4$ synthetic samples and run the bEM method until convergence (to double precision) to obtain the ML estimate $\mu^\star$. We compare the bEM, \SEM, \IEM, \SEMVR\ and \FIEM\ methods in terms of their precision measured by $| \mu - \mu^\star |^2$. We set the stepsize of the \SEM\ as $\gamma_k = 3/(k+10)$, and the stepsizes of the \SEMVR\ and the \FIEM\ to a constant stepsize proportional to $1/n^{2/3}$ and equal to $\gamma = 0.003$. The left plot of Figure \ref{fig:gmmplots} shows the convergence of the precision $|\mu - \mu^*|^2$ for the different methods against the epoch(s) elapsed (one epoch equals $n$ iterations). We observe that the \SEMVR\ and \FIEM\ methods outperform the other methods, supporting our analytical results.

\paragraph{Varying sample size} We compare the number of \emph{iterations} required to reach a precision of $10^{-3}$ as a function of the sample size from $n = 10^3$ to $n=10^5$. We average over 5 independent runs for each method using the same stepsizes as in the finite sample size case above.
The right plot of Figure \ref{fig:gmmplots} confirms that our findings in Theorem~\ref{thm:iem} and \ref{thm:svrg} are sharp. It requires ${\cal O}( n/\epsilon )$ (\resp ${\cal O}(n^{\frac{2}{3}}/\epsilon)$) iterations to find a $\epsilon$-stationary point for the \IEM\ (\resp \SEMVR\ and \FIEM) method.

\begin{figure}[H]
\centering
\includegraphics[width=.995\textwidth]{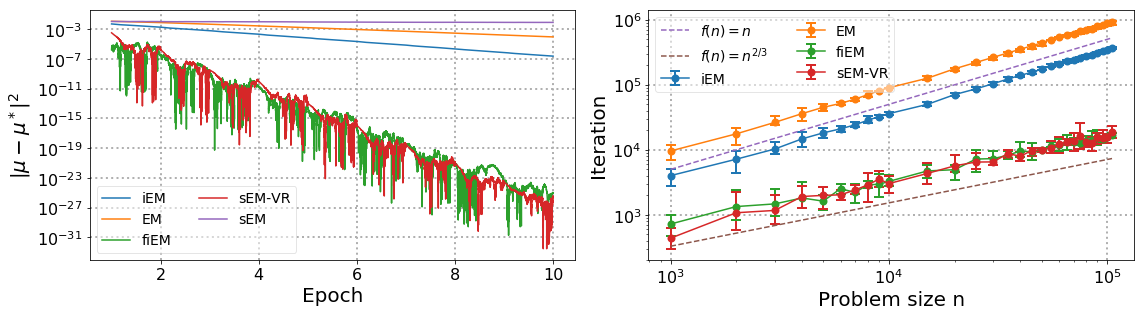}\vspace{-.3cm}
\caption{Performance of stochastic EM methods for fitting a GMM. (Left) Precision ($| \mu^{(k)} - \mu^\star |^2$) as a function of the epoch elapsed. (Right) Number of iterations to reach a precision of $10^{-3}$.}
\label{fig:gmmplots}
\end{figure}

\subsection{Probabilistic Latent Semantic Analysis}
The second example considers probabilistic Latent Semantic Analysis (pLSA) whose aim is to classify documents into a number of topics.
We are given a collection of documents $\inter[D]$ with terms from a vocabulary
$\inter[V]$. The data is summarized by a list of tokens $\{y_i \}_{i=1}^n$ where each token is a pair of document and word $y_i= (\obsdoc_i,\obsword_i)$ which indicates that $\obsword_i$ appears in document $\obsdoc_i$. The goal of pLSA is to classify the documents into $K$ topics, which is modeled as a latent variable $z_i \in \inter[K]$  associated with each token \citep{hofmann2017probabilistic}. 

To apply stochastic EM methods for pLSA,
we define $\param \eqdef (\pardoc,\partop)$ as the parameter variable, where $\pardoc =\{\pardoc_{d,k} \}_{\inter[K-1] \times \inter[D]}$ and $\partop= \{\partop_{k,v} \}_{\inter[K]\times \inter[V-1]}$. 
The constraint set $\Param$ is given as ---
for each $d \in \inter[D]$, $\pardoc_{d,\cdot} \in \Delta^{K}$ and for each $k \in \inter[K]$, we have $\partop_{\cdot,k} \in \Delta^{V}$, where $\Delta^{K}$, $\Delta^{V}$ are the (reduced dimension) $K,V$-dimensional probability simplex; see \eqref{eq:param_app} in the supplementary material for the precise definition.
Furthermore, denote $\pardoc_{d,K}= 1- \sum_{k=1}^{K-1} \pardoc_{d,k}$ for each $d \in \inter[D]$, and $\partop_{k,V}= 1 - \sum_{\ell=1}^{V-1} \partop_{k,\ell}$ for each $k \in \inter[K]$,
the complete log likelihood for $(y_i, z_i)$ is (up to an additive constant term):
\beq \label{eq:comp_like_plsa}
\log f(z_i,y_i;\param) = 
\sum_{k=1}^K\sum_{d=1}^D \log(\pardoc_{d,k}) \indiacc{k,d}(z_i,\obsdoc_i) +  \sum_{k=1}^K \sum_{v=1}^V \log (\partop_{k,v}) \indiacc{k,v}(z_i,\obsword_i).
\eeq
The penalization function is designed as
\beq
\Pen(\pardoc,\partop)= - \log \Dir(\pardoc; K, \alpha') - \log \Dir(\partop; V, \beta'),
\eeq
such that we ensure $\op( {\bm s}) \in {\rm int}( \Param)$. 
We can apply the stochastic EM methods described in Section~\ref{sec:sEM} on the pLSA problem. 
The implementation details are provided in Appendix~\ref{app:plsa}, therein we also verify the model assumptions required by our convergence analysis for pLSA.

\paragraph{Experiment} We compare the stochastic EM methods on two FAO (UN Food and Agriculture Organization) datasets \citep{medelyan2009human}. The first (\resp second) dataset consists of $10^3$ (\resp $10.5 \times 10^3$) documents and a vocabulary of size $300$. The number of topics is set to $K=10$ and the stepsizes for the \FIEM\ and \SEMVR\ are set to $\gamma =  1/n^{2/3}$ while the stepsize for the \SEM\ is set to $\gamma_k = 1/(k+10)$. 
Figure \ref{fig:gmmplots} shows the evidence lower bound (ELBO) as a function of the number of epochs for the datasets.
Again, the result shows that \FIEM\ and \SEMVR\ methods achieve faster convergence than the competing EM methods, affirming our theoretical findings.

\begin{figure}[H]
\includegraphics[width=\textwidth]{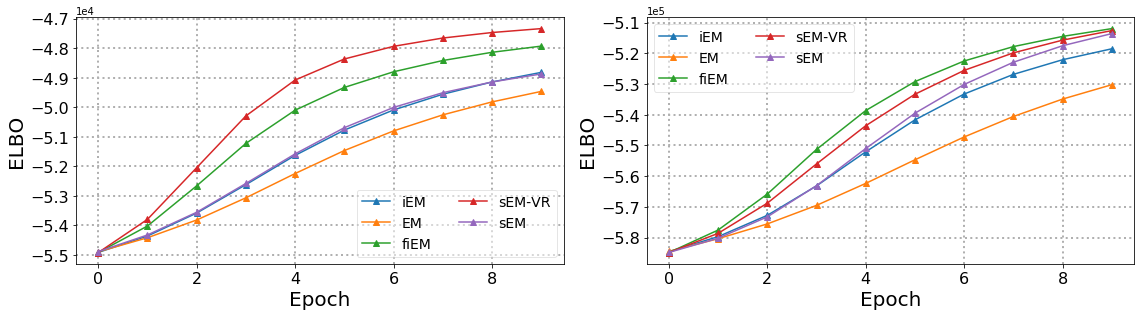}\vspace{-.2cm}
\caption{ELBO of the stochastic EM methods on FAO datasets as a function of number of epochs elapsed. (Left) small dataset with $10^3$ documents. (Right) large dataset with $10.5\times10^3$ documents.}\vspace{-.2cm}
\label{fig:plsaplots}
\end{figure}

\section{Conclusion}
This paper studies the global convergence for stochastic EM methods. Particularly, we focus on the inference of latent variable model with exponential family distribution and analyze the convergence of several stochastic EM methods. Our convergence results are \emph{global} and \emph{non-asymptotic}, and we offer two complimentary views on the existing stochastic EM methods --- one interprets \IEM\ method as an incremental MM method, and one interprets \SEMVR\ and \FIEM\ methods as scaled gradient methods. The analysis shows that the \SEMVR\ and \FIEM\ methods converge faster than the  \IEM\ method, and the result is confirmed via numerical experiments. 

\section*{Acknowledgement} 
BK and HTW contributed equally to this work. HTW's work is supported by the CUHK Direct Grant \#4055113.

\newpage
\linespread{1.1}
\normalsize

\bibliographystyle{abbrvnat}
\bibliography{references}

\newpage
\appendix
\section{Proof of Lemma~\ref{lem:lips_theta}}
\begin{Lemma*}
Assume H\ref{ass:compact}, H\ref{ass:regularity-phi-psi}, H\ref{ass:expected}, H\ref{ass:reg}. Let $e_i ( \param; \param' ) \eqdef Q_i( \param; \param' ) - {\cal L}_i( \param )$. For any $\param, \bar{\param}, \param' \in \Param^3$, we have
$\| \grd e_i( \param; \param' ) - \grd e_i( \bar{\param} ; \param' ) \| \leq \Lip{e} \| \param - \bar\param \| $, where $\Lip{e} \eqdef C_{\phi} C_{\Zset} \Lip{p} +  C_{\Sset} \Lip{\phi}$.
\end{Lemma*}
\begin{proof}
Observe the following identity
\beq \label{eq:fisher}
\begin{split}
\grd_{\param} \calL_i( \param) \big|_{\param = \hat{\param}}  & = \grd_{\param} \Big\{ - \log \int_{\sf Z} f( z_i, y_i ; \param ) \mu( \rmd z_i ) \Big\} \Big|_{\param = \hat{\param}} \\
& \overset{(a)}{=} - \int_{\sf Z} \{ \grd_{\param} \log f( z_i, y_i ; \param ) \} \big|_{\param = \hat{\param}} f( z_i | y_i; \hat\param ) \mu( \rmd z_i ) \\
& = \grd_{\param} \surrogate_i (\param ; \hat{\param} ) \big|_{\param = \hat{\param}} \vspace{.1cm}
\end{split}
\eeq
where (a) is due to the Fisher's identity and (b) is due to the definition of $\surrogate_i$ in \eqref{eq:definition-surrogate}.
It follows that
\beq
\nabla e_i( \param ; \hat{\param} ) = \nabla \{ \surrogate_i(\param,\hat{\param}) - \calL_i(\param) \}= \jacob{\phi}{\param}{\param}^\top \big( \overline{\bss}_i(\hat{\param}) - \overline{\bss}_i(\param) \big).
\eeq
We observe that
\beq
\begin{split}
\| \overline{\bss}_i ( \param ) - \overline{\bss}_i ( \param' ) \| & = \left\| \int_{\Zset} S(z_{i},y_i) \big\{ p(z_i|y_i; \param ) - p(z_i|y_i; \param' ) \big\} \mu(\rmd z_i) \right\| \\
& \leq \Lip{p} \| \param - \param' \| \int_{\Zset} | S(z_i,y_i) | \mu(\rmd z_i) \leq C_{\Zset} \Lip{p} \| \param - \param' \|
\end{split}
\eeq
where the last inequality is due to the compactness of $\Zset$.
Finally, we have
\beq
\begin{split}
\| \nabla e_i( \param ; \hat{\param} ) - \nabla e_i( \bar{\param} ; \hat{\param} ) \|
& \leq \| \jacob{\phi}{\param}{\param} \| \| \overline{\bss}_i ( \param ) - \overline{\bss}_i ( \bar{\param} ) \| + \| \jacob{\phi}{\param}{\param} - \jacob{\phi}{\param}{\bar\param} \|
\| \overline{\bss}_i ( \hat\param ) - \overline{\bss}_i ( \bar{\param} ) \| \\
& \leq \big( C_{\phi} C_{\Zset} \Lip{p} + C_{\Sset} \Lip{\phi} \big) \| \param - \bar{\param} \|
\end{split}
\eeq
where the last inequality is due to the compactness of $\Sset$.
\end{proof}

\section{Proof of Theorem~\ref{thm:iem}}
\begin{Theorem*}
Consider the \IEM\ algorithm, \ie Algorithm~\ref{alg:sem} with \eqref{eq:iem}.  Assume H\ref{ass:compact}, H\ref{ass:regularity-phi-psi}, H\ref{ass:expected}, H\ref{ass:reg}.
For any $K_{\max} \geq 1$, it holds that
\beq \notag
\EE[ \| \grd \overline\calL( \hp{K} ) \|^2 ] \leq n \!~  \frac{2 \Lip{e}}{K_{\sf max}}  \!~ \EE \big[ \overline\calL ( \hp{0} ) - \overline\calL ( \hp{K_{\sf max}} ) \big],
\eeq
where $\Lip{e}$ is defined in Lemma~\ref{lem:lips_theta} and $K$ is a uniform random variable on  $\inter[{0,K_{\max}-1}]$ [cf.~\eqref{eq:random}] independent of the $\{i_k\}_{k=0}^{K_{\max}}$.
\end{Theorem*}
\begin{proof}
We derive a \emph{non-asymptotic} convergence rate for the \IEM\ method. To begin our analysis, define
\beq
\overline\calL^{(k+1)} ( \param ) \eqdef \Pen( \param ) +  \frac{1}{n} \sum_{i=1}^n \surrogate_i ( \param ; \hat{\param}^{(\tau_i^{k+1})} )
\eeq
One has
\beq
\overline\calL^{(k+1)} ( \param ) = \overline\calL^{(k)} ( \param ) + { \frac{1}{n}} \big(  \surrogate_{i_k} ( \param ; \hat{\param}^{(k)} ) - \surrogate_{i_k} ( \param ; \hat{\param}^{(\tau_{i_k}^{k})} ) \big)
\eeq
Observe that $\hat{\param}^{(k+1)} \in \argmin_{\param \in \Param} \overline\calL^{(k+1)} ( \param )$. We have
\beq
\begin{split}
\overline\calL^{(k+1)} ( \hat\param^{(k+1)} ) \leq \overline\calL^{(k+1)} ( \hat\param^{(k)} ) & = \overline\calL^{(k)} ( \hat\param^{(k)} ) + { \frac{1}{n}} \big( \surrogate_{i_k} ( \hat\param^{(k)} ; \hat{\param}^{(k)} ) - \surrogate_{i_k} ( \hat\param^{(k)} ; \hat{\param}^{(\tau_{i_k}^{k})} ) \big) \\
& = \overline\calL^{(k)} ( \hat\param^{(k)} ) + { \frac{1}{n}} \big( \calL_{i_k} ( \hat\param^{(k)} ) - \surrogate_{i_k} ( \hat\param^{(k)} ; \hat{\param}^{(\tau_{i_k}^{k})} ) \big)
\end{split}
\eeq
where we have used the identity $\calL_{i_k} ( \hat\param^{(k)} ) = \surrogate_{i_k} ( \hat\param^{(k)} ; \hat{\param}^{(k)} )$.
Arranging terms imply
\beq \label{eq:ineq1}
e_{i_k} ( \hat\param^{(k)} ; \hat{\param}^{(\tau_{i_k}^{k})} ) =
\surrogate_{i_k} ( \hat\param^{(k)} ; \hat{\param}^{(\tau_{i_k}^{k})} ) -  \calL_{i_k} ( \hat\param^{(k)} ) \leq n \big( \overline\calL^{(k)} ( \hat\param^{(k)} ) - \overline\calL^{(k+1)} ( \hat\param^{(k+1)} ) \big)
\eeq
For $k \in \nset^*$, denote by $\mcf_{k}$ the $\sigma$-algebra generated by the random variables $i_0,\dots,i_{k-1}$. Note that $\hat{\param}^{(k)}$ is $\mcf_k$-measurable.
Because the random variable $i_k$ is independent of $\mcf_{k-1}$ and is uniformly distributed over $\{1,\dots,n\}$, the conditional expectation evaluates to
\beq
\CPE{e_{i_k} ( \hat\param^{(k)} ; \hat{\param}^{(\tau_{i_k}^{k})})}{\mcf_{k}}
= \overline\calL^{(k)} (\hat{\param}^{(k)}) - \overline\calL( \hat{\param}^{(k)} )
\eeq
where $\overline\calL$ is the global objective function defined in \eqref{eq:em_motivate}.
Note that the function $\overline\calL^{(k)} ( \param ) - \overline\calL( \param )$ is non-negative and $\Lip{e}$-smooth. It follows that for any $\param$, the inequality holds
\beq
\begin{split}
0 & \leq  \overline\calL^{(k)} ( \param ) - \overline\calL( \param ) \leq \overline\calL^{(k)} (\hat{\param}^{(k)}) - \overline\calL( \hat{\param}^{(k)} ) - \pscal{\grd \overline\calL ( \hat\param^{(k)} )}{ \param - \hat{\param}^{(k)} }+ \frac{\Lip{e}}{2} \| \param - \hat{\param}^{(k)} \|^2
\end{split},
\eeq
where we have used the fact $\grd \overline\calL^{(k)} ( \hat\param^{(k)} ) = {\bm 0}$. Setting $\param = \hat{\param}^{(k)} + (\Lip{e})^{-1} \grd \overline\calL ( \hat\param^{(k)} )$ in the above yields
\beq
\frac{1}{2 \Lip{e}} \| \grd \overline\calL ( \hat\param^{(k)} ) \|^2 \leq \overline\calL^{(k)} (\hat{\param}^{(k)}) - \overline\calL( \hat{\param}^{(k)} )
\eeq
Therefore, taking the conditional expectation on both sides of \eqref{eq:ineq1} leads to
\beq
\frac{1}{2n \Lip{e}} \| \grd \overline\calL ( \hat\param^{(k)} ) \|^2 \leq
\overline\calL^{(k)} ( \hat\param^{(k)} ) - \EE \big[  \overline\calL^{(k+1)} ( \hat\param^{(k+1)} ) | {\cal F}_k \big]
\eeq
Note that as we have set $\gamma_{k+1} = 1$ in the \IEM\ method, the terminating iteration number $K$ is chosen uniformly over $\{1,\dots,K_{\sf max}\}$, therefore taking the total expectations gives
\beq
\begin{split}
 \EE \big[ \| \grd \overline\calL( \hat{\param}^{(K)} ) \|^2 \big] & = \frac{1}{K_{\sf max}} \sum_{k=0}^{K_{\sf max}-1} \EE \big[\| \grd \overline\calL ( \hat\param^{(k)} ) \|^2 \big] \\
& \leq \frac{2n \Lip{e}}{K_{\sf max}} \EE \Big[
\overline\calL^{(0)} ( \hat\param^{(0)} ) -  \overline\calL^{(K_{\sf max})} ( \hat\param^{(K_{\sf max}+1)} ) \Big] \\
& \leq \frac{2n \Lip{e} }{K_{\sf max}} \EE \Big[
\overline\calL^{(0)} ( \hat\param^{(0)} ) -  \overline\calL ( \hat\param^{(K_{\sf max})} ) \Big]
\end{split}
\eeq
Lastly, we note that
$\overline\calL ( \hat\param^{(0)} ) = \overline\calL^{(0)} ( \hat\param^{(0)} )$.
This leads to \eqref{eq:iem_bdd} and concludes our proof.
\end{proof}
\section{Proof of Lemma~\ref{lem:global}}
\begin{Lemma*}
For any $\bss \in \Sset$, it holds that 
\beq \notag
\grd_{ \bss} V( {\bss} )  = \jacob{\overline{\param}}{\bss}{\bss}^\top \grd_\param \overline\calL( \overline\param( {\bss} ) ).
\eeq
Assume H\ref{ass:reg}. If $\bss^\star \in \set{ \bss \in \Sset }{\grd_{\bss} V(\bss )=0}$, then $\overline{\param}(\bss^\star) \in \set{\param \in \Param }{\grd_\param \overline{\calL}(\param)=0}$. Conversely, if $\param^* \in \set{\param \in \Param}{\grd_\param \overline{\calL}(\param)=0}$, then $\bss^*= \overline{\bss}(\param^*) \in \set{ \bss \in \Sset }{\grd_{\bss} V(\bss)=0}$.
\end{Lemma*}
\begin{proof}
Using chain rule, we obtain  $\grd_{ \bss} V( {\bss} )  = \jacob{\overline{\param}}{\bss}{\bss}^\top \grd_\param \overline\calL( \overline\param( {\bss} ) )$
Obviously if $\grd_{ \bss} V( {\bss^\star} ) = 0$, then $\grd_\param \overline\calL( \overline\param( {\bss^\star} ) ) = 0$ because $\jacob{\overline{\param}}{\bss}{\bss}$ is invertible.
Consider now the converse. By the Fisher identity, we get $\grd_\param \calL_i(\param)= \grd_\param \psi(\param) - \jacob{\phi}{\param}{\param}^\top \overline{\bss}_i(\param)$ which implies that $\grd_\param \overline{\calL}(\param)= \grd_\param \Pen(\param) + \grd_\param \psi(\param) - \jacob{\phi}{\param}{\param}^{\top} \overline{\bss}(\param)$.
Hence, if $\grd_\param \overline{\calL} (\param^*)=0$, then $\grd_\param \Pen(\param^*) + \grd_\param \psi(\param^*) - \jacob{\phi}{\param}{\param^*}^{\top} \overline{\bss}^* = 0$ where we have set $\overline{\bss}^*= \overline{\bss}(\param^*)$. Under H\ref{ass:reg}, the latter relation implies that $\param^*= \overline{\param}(\bss^*)$. The proof follows.
\end{proof}

\section{Proof of Lemma~\ref{lem:semigrad}}
\begin{Lemma*} 
Assume H\ref{ass:reg},H\ref{ass:eigen}. For all $\bss \in \Sset$,
\beq \notag
\upsilon_{\min}^{-1} \pscal{\grd V ( {\bss} ) }{ {\bss} - \os( \op ({\bss})) }
\geq \big\| {\bss} - \os( \op ({\bss})) \big\|^2 \geq \upsilon_{\max}^{-2} \| \grd V ( {\bss} ) \|^2,
\eeq
\end{Lemma*}
\begin{proof}
Using H\ref{ass:reg} and the fact that we can exchange integration with differentiation and the Fisher's identity,   we obtain
\beq \label{eq:grd_v}
\begin{split}
\grd_{ \bss} V( {\bss} ) & = \jacob{ \overline{\param} }{ \bss }{\bss}^\top
\Big( \grd_\param \Pen( \mstep{\bss} )  + \grd_\param \calL( \overline\param( {\bss} ) )  \Big) \\
& =  \jacob{ \overline{\param} }{ \bss }{\bss}^\top \Big( \grd_\param \psi( \mstep{\bss}) + \grd_\param \Pen( \mstep{\bss} ) - \jacob{\phi}{\param}{\mstep{\bss} }^\top  \os( \op ({\bss})) \Big)\\
& =   \jacob{ \overline{\param} }{ \bss }{\bss}^\top \jacob{\phi}{\param}{ \mstep{\bss} }^\top \!~ ({\bss} - \os( \op ({\bss})) ) \eqsp,
\end{split}
\eeq
Consider the following vector map:
\beq
{\bss} \to \grd_{\param} L(\bss, \param) \vert_{\param= \mstep{\bss}}= \grd_\param \psi ( \mstep{\bss} ) + \grd_{ \param} \Pen(\mstep{\bss}  ) - \jacob{ \phi }{ \param }{\mstep{\bss}  }^\top \!~{\bss} \eqsp.
\eeq
Taking the gradient of the above map \wrt ${\bss}$ and using assumption H\ref{ass:reg}, we show that:
\beq
{\bm 0} = - \jacob{\phi}{\param}{\mstep{\bss} } + \Big( \underbrace{ \grd_{\param}^2 \big( \psi( \param ) + \Pen( \param ) - \pscal{ \phi( \param ) }{ {\bss} } \big)}_{= \hess{{L}}{\param} ( {\bss}; \param )} \big|_{\param = \mstep{\bss}  } \Big) \jacob{ \overline{\param} }{\bss}{\bss} \eqsp.
\eeq
The above yields
\beq
\grd_{ \bss} V( {\bss} )  = \operatorname{B}(\bss) ({\bss} - \os( \op ({\bss})) )
\eeq
where we recall $\operatorname{B}(\bss) = \jacob{ \phi }{ \param }{ \mstep{\bss} } \Big( \hess{{L}}{\param}( {\bss}; \mstep{\bss} )  \Big)^{-1} \jacob{ \phi }{ \param }{\mstep{\bss} }^\top$. The proof of \eqref{eq:semigrad} follows directly from the assumption~H\ref{ass:eigen}.
\end{proof}

\section{Proof of Lemma~\ref{lem:smooth}}
\begin{Lemma*} 
Assume H\ref{ass:compact}, H\ref{ass:expected}, H\ref{ass:reg}, H\ref{ass:eigen}.  
For all $\bss,\bss' \in \Sset$ and $i \in \inter$, we have
\beq \notag
\| \overline{\bss}_i ( \overline{\param} ({\bss})) - \overline{\bss}_i ( \overline{\param} ({\bss}' )) \| \leq \Lip{{\bss}} \| {\bss} - {\bss}' \|,~~\| \grd  V ( {\bss} ) - \grd  V ( {\bss}' ) \| \leq \Lip{V} \| {\bss} - {\bss}' \|,
\eeq
where $\Lip{\bss} \eqdef C_{\Zset} \Lip{p} \Lip{\theta}$ and $\Lip{V}  \eqdef \upsilon_{\max} \big( 1 + \Lip{{\bss}} \big) + \Lip{B} C_{\Sset}$.
\end{Lemma*}
\begin{proof}
We prove the first inequality of the lemma in \eqref{eq:smooth}. Observe that
\beq
\overline{\bss}_i( \mstep{\bss} ) - \overline{\bss}_i( \mstep{\bss'} )
= \int_{\Zset} S(z_{i},y_i) \big\{ p(z_i|y_i; \mstep{\bss} ) - p(z_i|y_i; \mstep{\bss'} ) \big\} \mu(\rmd z_i)
\eeq
Taking norms on both sides and using H\ref{ass:compact}, H\ref{ass:expected} yield
\beq \label{eq:lemsi}
\| \overline{\bss}_i( \mstep{\bss} ) - \overline{\bss}_i( \mstep{\bss'} )  \|
\leq \Lip{p} \| \mstep{\bss} - \mstep{\bss'} \| \int_{\Zset} | S(z_{i},y_i) | \mu(\rmd z_i)
\leq C_{\Zset} \Lip{p} \| \mstep{\bss} - \mstep{\bss'} \|,
\eeq
where we have $\int_{\Zset} | S(z_{i},y_i) | \mu(\rmd z_i) \leq C_{\Zset}$. Furthermore, under H\ref{ass:reg}, as $\mstep{\bss}$ is Lipschitz, there exists $\Lip{\theta}$ such that
\beq
\| \mstep{\bss} - \mstep{\bss'} \| \leq \Lip{\theta} \| \bss - \bss' \|
\eeq
Substituting back into \eqref{eq:lemsi} concludes the proof with $\Lip{\bss} = C_{\Zset} \Lip{p} \Lip{\theta}$.

To prove the second inequality in \eqref{eq:smooth}, we observe that:
\beq
\grd_{ \bss} V( {\bss} )  =  \operatorname{B}(\bss)\!~ ({\bss} - \os( \op ({\bss})) )
\eeq
We observe the upper bound
\beq
\begin{split}
& \| \grd V( {\bss} )  - \grd V ({\bss}') \| \\
&  = \| \operatorname{B}(\bss) ( ({\bss} - \os( \op ({\bss})) )  -  ({\bss}' - \os( \op ({\bss}')) )  ) + \big( \operatorname{B}(\bss)  - \operatorname{B}(\bss') \big)  ({\bss}' - \os( \op ({\bss}')) )  \| \\
& \leq \| \operatorname{B}(\bss) \| \| {\bss} - \os( \op ({\bss}))  -  ({\bss}' - \os( \op ({\bss}')) )  \| + \|  \operatorname{B}(\bss)  - \operatorname{B}(\bss') \| \| {\bss}' - \os( \op ({\bss}')) \|
\end{split}
\eeq
We observe that
\beq
\| \os( \op ({\bss})) - \os( \op ({\bss}')) \| \leq \frac{1}{n} \sum_{i=1}^n \| \overline{\bss}_i( \mstep{\bss} ) - \overline{\bss}_i( \mstep{\bss'} )  \| \leq \Lip{{\bss}} \| \bss - \bss' \|,
\eeq
which is due to \eqref{eq:smooth}. Furthermore, as $\bss' \in \Sset$, a compact set, we have $\| {\bss}' - \os( \op ({\bss}')) \|  \leq C_{\Sset}$.
Consequently, using H\ref{ass:eigen} we have
\beq
\begin{split}
& \| \grd V( {\bss} )  - \grd V ({\bss}') \| \leq \Big( \upsilon_{\max} \big( 1 + \Lip{{\bss}} \big) +  \Lip{B} C_{\Sset} \Big) \| \bss - \bss' \|,
\end{split}
\eeq
which proves our claim.
\end{proof}
\section{Proof of Theorem~\ref{thm:svrg}}
\begin{Theorem*}
Assume H\ref{ass:compact}, H\ref{ass:expected}, H\ref{ass:reg}, H\ref{ass:eigen}. Denote $\overline{L}_{\sf v} = \max\{ \Lip{V}, \Lip{\bss} \}$ with the constants in Lemma~\ref{lem:smooth}.
\begin{itemize}[leftmargin=5.5mm]
\item Consider the \SEMVR\ method, \ie Algorithm~\ref{alg:sem} with  \eqref{eq:svrgem}. There exists a universal constant $\mu \in (0,1)$ (independent of $n$) such that if we set the step size as $\gamma = \frac{\mu \upsilon_{\min} }{ \overline{L}_{\sf v} n^{2/3}}$ and the epoch length as $m = \frac{n }{2 \mu^2 \upsilon_{\min}^2 +\mu }$, then for any $K_{\sf max} \geq 1$ that is a multiple of $m$, it holds that
\beq \notag
\EE[ \| \grd V( \hs{K} ) \|^2 ] \leq  n^{\frac{2}{3}} \!~ \frac{2 \overline{L}_{\sf v} }{\mu K_{\sf max}} \frac{ \upsilon_{\max}^2 }{ \upsilon_{\min}^2 } \!~ \EE[ V( \hs{0} ) - V( \hs{K_{\sf max}}) ] .
\eeq
\item Consider the \FIEM\ method, \ie Algorithm~\ref{alg:sem} with  \eqref{eq:sagaem}. Set $\gamma = \frac{\upsilon_{\min}}{\alpha \overline{L}_{\sf v} n^{2/3}}$ such that $\alpha = \max\{ 6, 1 + 4 \upsilon_{\min} \}$. For any $K_{\sf max} \geq 1$, it holds that
\beq \notag
\EE[ \| \grd V( \hs{K} ) \|^2 ] \leq n^{\frac{2}{3}} \!~ \frac{ \alpha^2 \overline{L}_{\sf v} }{K_{\sf max}}  \frac{ \upsilon_{\max}^2 }{ \upsilon_{\min}^2 }\EE \big[ V( \hs{0} ) - V( \hs{K_{\sf max}} ) \big] .
\eeq
\end{itemize}
We recall that $K$ in the above is a uniform and independent r.v.~chosen from $\inter[K_{\sf max}-1]$ [cf.~\eqref{eq:random}].
\end{Theorem*}
To simplify notation, we shall denote $c_1 = \upsilon_{\min}^{-1}$ and $d_1 = \upsilon_{\max}$ in the below.

\paragraph{Proof for the \SEMVR\ method} We first establish the following auxiliary lemma:
\begin{Lemma}\label{lem:aux2}
For any $k \geq 0$ and consider the update in \eqref{eq:svrgem}, it holds that
\beq
\EE[ \| \hs{k} - \StocEstep^{(k+1)} \|^2 ] \leq 2 \EE[ \| \hs{k} - \os^{(k)} \|^2 ] +  2 \Lip{\bss}^2  \EE[ \| \hs{k} -  \hs{ \ell(k) } \|^2 ] ,
\eeq
where we recall that $\ell(k) \eqdef m \lfloor \frac{k}{m} \rfloor$ is the first iteration number in the epoch that iteration $k$ is in.
\end{Lemma}
\begin{proof}
We observe that
\beq \label{eq:auxlem2}
\EE[ \| \hs{k} - \StocEstep^{(k+1)} \|^2 ] \leq 2 \EE[ \| \hs{k} - \os^{(k)} \|^2] + 2 \EE[ \| \os^{(k)} - \StocEstep^{(k+1)} \|^2 ]
\eeq
For the latter term, we obtain its upper bound as 
\beq
\begin{split}
\EE[ \| \os^{(k)} - \StocEstep^{(k+1)} \|^2 ] & = \EE\Big[ \Big\| \frac{1}{n} \sum_{i=1}^n \big( \os_i^{(k)} - \os_i^{\ell(k)} \big) - \big( \os_{i_k}^{(k)} - \os_{i_k}^{(\ell(k))} \big) \Big\|^2 \Big] \\
& \leq \EE[ \| \os_{i_k}^{(k)} - \os_{i_k}^{(\ell(k))} \|^2 ] \leq \Lip{\bss}^2 \EE[ \| \hs{k} - \hs{\ell(k)} \|^2 ]
\end{split}
\eeq
Substituting into \eqref{eq:auxlem2} proves the lemma.
\end{proof}
To proceed with our proof, we shall consider a constant step size $\gamma_k = \gamma$ and observe that
\beq
\begin{split}
V( \hs{k+1} ) & \leq V( \hs{k} ) - \gamma \pscal{ \hs{k} - \StocEstep^{(k+1)} }{ \grd V( \hs{k} ) } + \frac{\gamma^2 \Lip{V}}{2} \| \hs{k} - \StocEstep^{(k+1)} \|^2 \\
\end{split}
\eeq
Using \eqref{eq:expected_grad} and taking expectations on both sides show that
\beq \label{eq:lips_con}
\begin{split}
& \EE[ V( \hs{k+1} ) ] \\
 & \leq \EE[ V( \hs{k} ) ] - \gamma \EE \Big[ \pscal{ \hs{k} - \os^{(k)} }{\grd V( \hs{k} ) } \Big]
+ \frac{\gamma^2 \Lip{V}}{2} \EE[ \| \hs{k} - \StocEstep^{(k+1)} \|^2 ] \\
& \overset{(a)}{\leq}  \EE[ V( \hs{k} ) ] - \frac{\gamma}{c_1} \EE[ \| \hs{k} - \os^{(k)} \|^2 ] + \frac{\gamma^2 \Lip{V}}{2} \EE[ \| \hs{k} - \StocEstep^{(k+1)} \|^2 ] \\
\end{split}
\eeq
where (a) is due to Lemma~\ref{lem:semigrad}.
Furthermore, for $k+1 \leq \ell(k) + m$ (\ie $k+1$ is in the same epoch as $k$), we have
\beq
\begin{split}
& \EE[ \| \hs{k+1} -  \hs{\ell(k)} \|^2 ] = \EE[ \| \hs{k+1} - \hs{k} + \hs{k} - \hs{\ell(k)} \|^2 ] \\
& = \EE \Big[  \| \hs{k} -  \hs{\ell(k)} \|^2 + \| \hs{k+1} - \hs{k}  \|^2 + 2 \pscal{\hs{k} -  \hs{\ell(k)} }{\hs{k+1} - \hs{k} } \Big] \\
& = \EE \Big[ \| \hs{k} -  \hs{\ell(k)} \|^2 + \gamma^2 \| \hs{k} - \StocEstep^{(k+1)} \|^2 -2\gamma \pscal{\hs{k} -  \hs{\ell(k)} }{ \hs{k} - \os^{(k)}  } \Big] \\
& \leq \EE \Big[ (1 + \gamma \beta) \| \hs{k} -  \hs{\ell(k)} \|^2 + \gamma^2 \| \hs{k} - \StocEstep^{(k+1)} \|^2 + \frac{\gamma}{\beta} \| \hs{k} - \os^{(k)} \|^2 \Big],
\end{split}
\eeq
where the last inequality is due to the Young's inequality.
Consider the following sequence
\beq
R_k \eqdef \EE[ V( \hs{k} ) + b_{{k}} \| \hs{k} - \hs{\ell(k)} \|^2 ]
\eeq
where $b_k \eqdef \overline{b}_{k~{\rm mod}~m}$ is a periodic sequence where:
\beq
\overline{b}_i = \overline{b}_{i+1} (1 + \gamma \beta + 2 \gamma^2 \Lip{\bss}^2 ) + \gamma^2 \Lip{V} \Lip{\bss}^2,~~i=0,1,\dots,m-1~~\text{with}~~\overline{b}_m = 0.
\eeq
Note that $\overline{b}_i$ is decreasing with $i$ and this implies
\beq
\overline{b}_i \leq \overline{b}_0 = \gamma^2 \Lip{V} \Lip{\bss}^2 \frac{ (1 + \gamma \beta + 2 \gamma^2 \Lip{\bss}^2 )^m - 1 }{ \gamma \beta + 2 \gamma^2 \Lip{\bss}^2 },~i=1,2,\dots,m.
\eeq
For $k+1 \leq \ell(k) + m$, we have the following inequality
\beq\begin{split}
R_{k+1 } & \leq  \EE \Big[ V( \hs{k} )  - \frac{\gamma}{c_1} \| \hs{k} - \os^{(k)} \|^2 + \frac{\gamma^2 \Lip{V}}{2} \| \hs{k} - \StocEstep^{(k+1)} \|^2 \Big] \\
& + b_{k+1} \EE \Big[ (1 + \gamma \beta) \| \hs{k} -  \hs{\ell(k)} \|^2 + \gamma^2 \| \hs{k} - \StocEstep^{(k+1)} \|^2 + \frac{\gamma}{\beta} \| \hs{k} - \os^{(k)} \|^2 \Big] \\
& \overset{(a)}{\leq}
\EE \Big[ V( \hs{k} )  - \big(  \frac{\gamma}{c_1} - \frac{b_{k+1} \gamma}{\beta} \big) \| \hs{k} - \os^{(k)} \|^2 + b_{k+1} (1 + \gamma \beta ) \| \hs{k} - \hs{\ell(k)} \|^2 \Big] \\
& + \Big( \gamma^2 \Lip{V} + 2 b_{k+1} \gamma^2 \Big) \EE \Big[ \| \hs{k} - \os^{(k)} \|^2 + \Lip{\bss}^2 \| \hs{k} - \hs{\ell(k)} \|^2 \Big],
\end{split} \eeq
where (a) is due to Lemma~\ref{lem:aux2}. Rearranging terms gives
\beq
\begin{split}
R_{k+1 } & \leq
\EE [ V( \hs{k} ) ] - \big(  \frac{\gamma}{c_1} - \frac{b_{k+1} \gamma}{\beta} -
\gamma^2 ( \Lip{V} + 2 b_{k+1} ) \big) \EE[ \| \hs{k} - \os^{(k)} \|^2 ] \\
& + \Big(  \underbrace{b_{k+1} (1 + \gamma \beta + 2 \gamma^2 \Lip{\bss}^2 ) + \gamma^2 \Lip{V} \Lip{\bss}^2}_{=b_k~~\text{since $k+1 \leq \ell(k)+m$}} \Big) \EE\Big[  \| \hs{k} - \hs{\ell(k)} \|^2 \Big] \\
& = R_k - \big(  \frac{\gamma}{c_1} - \frac{b_{k+1} \gamma}{\beta} -
\gamma^2 ( \Lip{V} + 2 b_{k+1} ) \big) \EE[ \| \hs{k} - \os^{(k)} \|^2 ]
\end{split}
\eeq
This leads to, for any $\gamma$ and $\beta$ such that $(1 - c_1 b_{k+1} \beta^{-1} -
c_1 \gamma ( \Lip{V} + 2 b_{k+1} )  > 0$,
\beq
\frac{1}{d_1^2} \EE[ \| \grd V( \hs{k} ) \|^2 ]  \leq \EE[ \| \hs{k} - \os^{(k)} \|^2 ] \leq \frac{c_1 ( R_{k} - R_{k+1} ) }{ \gamma \Big( 1 - c_1 b_{k+1} \beta^{-1} -
c_1 \gamma ( \Lip{V} + 2 b_{k+1} ) \Big)}.
\eeq
By setting $\beta = \frac{c_1 \overline{L}_{\sf v}}{n^{1/3}}$, $\gamma = \frac{\mu}{ c_1 \overline{L}_{\sf v}  n^{2/3}}$, $m = \frac{n c_1^2}{2 \mu^2+\mu c_1^2}$, it can be shown that there exists $\mu \in (0,1)$, such that the following lower bound holds
\beq
\begin{split}
& 1 - c_1 \gamma \Lip{V} - \big( \frac{c_1}{\beta}
 + 2 c_1 \gamma \big) b_{k+1}
 \geq 1 - \frac{\mu}{n^{\frac{2}{3}}} - \overline{b}_0 \big( \frac{n^{\frac{1}{3}}}{\overline{L}_{\sf v}} + \frac{2 \mu}{\overline{L}_{\sf v} n^{\frac{2}{3}}} \big) \\
 & \geq 1 - \frac{\mu }{n^{\frac{2}{3}}} - \frac{ \Lip{V} \mu^2 }{c_1^2 n^{\frac{4}{3}}} \frac{ (1 + \gamma \beta + 2 \gamma^2 \Lip{\bss}^2 )^m - 1 }{ \gamma \beta + 2 \gamma^2 \Lip{\bss}^2 } \big( \frac{n^{\frac{1}{3}}}{\overline{L}_{\sf v}} + \frac{2 \mu}{\overline{L}_{\sf v} n^{\frac{2}{3}}} \big) \\
 & \overset{(a)}{\geq} 1 - \frac{\mu}{ n^{\frac{2}{3}}} - \frac{ \mu }{c_1^2 } (\rme-1) \big( 1 + \frac{2 \mu}{n} \big)
 \geq 1 - \mu - \mu(1+2 \mu) \frac{\rme-1}{c_1^2} \overset{(b)}{ \geq} \frac{1}{2}
 \end{split}
\eeq
where the simplification in (a) is due to
\beq
\frac{\mu}{n} \leq \gamma \beta + 2 \gamma^2 \Lip{\bss}^2 \leq \frac{\mu}{n} + \frac{2 \mu^2}{c_1^2 n^{\frac{4}{3}}} \leq \frac{\mu c_1^2 + 2 \mu^2}{c_1^2} \frac{1}{n}~~\text{and}~~(1 + \gamma \beta + 2 \gamma^2 \Lip{\bss}^2 )^m \leq \rme-1.
\eeq
and the required $\mu$ in (b) can be found by solving the quadratic equation\footnote{In fact, for small $c_1$, this gives $\mu = \Theta(c_1)$}.
This gives
\beq
\EE[ \| \grd V( \hs{K} ) \|^2 ] = \frac{1}{K_{\sf max}} \sum_{k=0}^{K_{\sf max}-1}
\EE[ \| \grd V( \hs{k} ) \|^2 ] \leq \frac{2 d_1^2 c_1 ( R_0 - R_{K_{\sf max}} ) }{ \gamma  K_{\sf max}}
\eeq
Note that $R_0 = \EE[ V( \hs{0} ) ]$ and if $K_{\sf max}$ is a multiple of $m$, then $R_{\sf max} = \EE[ V( \hs{K_{\sf max}}) ]$. Under the latter condition, we have
\beq
\EE[ \| \grd V( \hs{K} ) \|^2 ] \leq
 n^{\frac{2}{3}} \!~ \frac{2 d_1^2 c_1^2 \overline{L}_{\sf v}}{\mu K_{\sf max}} \!~ \EE[ V( \hs{0} ) - V( \hs{K_{\sf max}}) ].
\eeq
This concludes our proof.

\paragraph{Proof for the \FIEM\ method} Our proof proceeds by  observing the following auxiliary lemma:
\begin{Lemma}\label{lem:aux1}
For any $k \geq 0$ and consider the update in \eqref{eq:sagaem}, it holds that
\beq
\EE[ \| \hs{k} - \StocEstep^{(k+1)} \|^2 ] \leq 2 \EE[ \| \hs{k} - \os^{(k)} \|^2 ] + \frac{2 \Lip{\bss}^2}{n} \sum_{i=1}^n \EE[ \|
\hs{k} -  \hs{t_i^k} \|^2 ]
\eeq
\end{Lemma}
\begin{proof}
We observe that $ \overline{\StocEstep}^{(k)} = \frac{1}{n} \sum_{i=1}^n \os_i^{(t_i^k)}$ and $\EE[\os_{i_k}^{(k)} - \os_{i_k}^{(t_{i_k}^k)} ] = \os^{(k)} - \overline{\StocEstep}^{(k)}$. Moreover, we recall that $ \os_i^{(k)} = \os_i( \hp{k} ) = \os_i( \op( \hs{k} ) )$. Thus
\beq
\begin{split}
& \EE[ \| \hs{k} - \StocEstep^{(k+1)} \|^2 ] \overset{(a)}{=} \EE[ \| \hs{k} - \os^{(k)} + (\os^{(k)} -\overline{\StocEstep}^{(k)}) - ( \os_{i_k}^{(k)} - \os_{i_k}^{(t_{i_k}^k)} )   \|^2 ] \\
& \leq 2 \EE[ \| \hs{k} - \os^{(k)} \|^2 ] + 2 \EE[ \| (\os^{(k)} -\overline{\StocEstep}^{(k)}) - ( \os_{i_k}^{(k)} - \os_{i_k}^{(t_{i_k}^k)} )   \|^2 ] \\
& \overset{(b)}{\leq} 2 \EE[ \| \hs{k} - \os^{(k)} \|^2 ] + 2 \EE[ \| \os_{i_k}^{(k)} - \os_{i_k}^{(t_{i_k}^k)}   \|^2 ],
\end{split}
\eeq
where (a) uses the SAGA update in \eqref{eq:sagaem}; (b) uses the variance inequality
$\EE[ \| X - \EE[X] \|^2 ] \leq \EE[ \| X \|^2 ]$ .
The last expectation can be further bounded by
\beq
\begin{split}
&
\EE[ \| \os_{i_k}^{(k)} - \os_{i_k}^{(t_{i_k}^k)} \|^2 ] = \frac{1}{n} \sum_{i=1}^n \EE[ \| \os_i^{(k)} - \os_i^{(t_i^k)} \|^2 ] \overset{(a)}{\leq} \frac{\Lip{\bss}}{n}
\sum_{i=1}^n \EE[ \| \hs{k} - \hs{t_i^k} \|^2 ],
\end{split}
\eeq
where (a) is due to Lemma~\ref{lem:semigrad}.
Combining the two equations above yields the desired lemma.
\end{proof}
Let $\gamma_{k+1} = \gamma$, \ie with a fixed step size. We observe the following
\beq
\begin{split}
V( \hs{k+1} ) & \leq V( \hs{k} ) - \gamma \pscal{ \hs{k} - \StocEstep^{(k+1)} }{ \grd V( \hs{k} ) } + \frac{\gamma^2 \Lip{V}}{2} \| \hs{k} - \StocEstep^{(k+1)} \|^2 \\
\end{split}
\eeq
Taking expectations on both sides yields
\beq \label{eq:lips_con}
\begin{split}
& \EE[ V( \hs{k+1} ) ] \\
 & \leq \EE[ V( \hs{k} ) ] - \gamma \EE \Big[ \pscal{ \hs{k} - \os^{(k)} }{\grd V( \hs{k} ) } \Big]
+ \frac{\gamma^2 \Lip{V}}{2} \EE[ \| \hs{k} - \StocEstep^{(k+1)} \|^2 ] \\
& \overset{(a)}{\leq}  \EE[ V( \hs{k} ) ] - \frac{\gamma}{c_1} \EE[ \| \hs{k} - \os^{(k)} \|^2 ] + \frac{\gamma^2 \Lip{V}}{2} \EE[ \| \hs{k} - \StocEstep^{(k+1)} \|^2 ] \\
& \overset{(b)}{\leq} \EE[ V( \hs{k} ) ] - \Big( \frac{\gamma}{c_1} - \gamma^2 \Lip{V} \Big) \EE[ \| \hs{k} - \os^{(k)} \|^2 ] + \frac{\gamma^2 \Lip{V} \Lip{\bss}^2}{n} \sum_{i=1}^n \EE[ \| \hs{k} -  \hs{t_i^k} \|^2 ]
\end{split}
\eeq
where (a) is due to Lemma~\ref{lem:semigrad} and (b) is due to Lemma~\ref{lem:aux1}.
Next, we observe that
\beq
\frac{1}{n} \sum_{i=1}^n \EE[ \| \hs{k+1} - \hs{t_i^{k+1}} \|^2 ] = \frac{1}{n} \sum_{i=1}^n
\Big( \frac{1}{n} \EE[ \| \hs{k+1} - \hs{k} \|^2 ] + \frac{n-1}{n} \EE[ \| \hs{k+1} - \hs{t_i^k} \|^2 ]  \Big)
\eeq
where the equality holds as $i_k$ and $j_k$ are drawn independently. For any $\beta > 0$, it holds
\beq
\begin{split}
& \EE[ \| \hs{k+1} - \hs{t_i^k} \|^2 ] \\
& = \EE \Big[ \| \hs{k+1} - \hs{k} \|^2 + \| \hs{k} - \hs{t_i^k} \|^2 + 2 \pscal{\hs{k+1} - \hs{k}}{\hs{k}- \hs{t_i^k}} \Big] \\
& = \EE \Big[ \| \hs{k+1} - \hs{k} \|^2 + \| \hs{k} - \hs{t_i^k} \|^2 - 2 \gamma \pscal{
\hs{k} - \os^{(k)} }{\hs{k}- \hs{t_i^k}} \Big] \\
& \leq  \EE \Big[ \| \hs{k+1} - \hs{k} \|^2 + \| \hs{k} - \hs{t_i^k} \|^2 +  \frac{\gamma}{\beta} \| \hs{k} - \os^{(k)} \|^2 + \gamma \beta \| \hs{k}- \hs{t_i^k} \|^2 \Big]
\end{split}
\eeq
where the last inequality is due to the Young's inequality. Subsequently, we have
\beq
\begin{split}
& \frac{1}{n} \sum_{i=1}^n \EE[ \| \hs{k+1} - \hs{t_i^{k+1}} \|^2 ] \\
& \leq \EE[  \| \hs{k+1} - \hs{k} \|^2 ] + \frac{n-1}{n^2} \sum_{i=1}^n \EE \Big[ (1+\gamma \beta) \|  \hs{k} - \hs{t_i^k} \|^2 + \frac{\gamma}{\beta} \| \hs{k} - \os^{(k)} \|^2 \Big]
\end{split}
\eeq
Observe that $\hs{k+1} - \hs{k} = - \gamma ( \hs{k} - \StocEstep^{(k+1)} )$. Applying Lemma~\ref{lem:aux1} yields
\beq \notag
\begin{split}
& \frac{1}{n} \sum_{i=1}^n \EE[ \| \hs{k+1} - \hs{t_i^{k+1}} \|^2 ] \\
& \leq \Big( 2 \gamma^2 + \frac{n-1}{n}\frac{\gamma}{\beta} \Big) \EE[  \| \hs{k} - \os^{(k)} \|^2 ]
+ \sum_{i=1}^n \Big( \frac{2 \gamma^2 \Lip{\bss}^2}{n} + \frac{(n-1)(1+\gamma \beta)}{n^2}\Big) \EE[ \| \hs{k} - \hs{t_i^k} \|^2 ] \\
& \leq \Big( 2 \gamma^2 +\frac{\gamma}{\beta} \Big) \EE[  \| \hs{k} - \os^{(k)} \|^2 ]
+ \sum_{i=1}^n \frac{1 - \frac{1}{n} +\gamma \beta + 2 \gamma^2 \Lip{\bss}^2 }{n}  \EE[ \| \hs{k} - \hs{t_i^k} \|^2 ]
\end{split}
\eeq
Let us define
\beq
\Delta^{(k)} \eqdef \frac{1}{n} \sum_{i=1}^n \EE[ \| \hs{k} - \hs{t_i^{k}} \|^2 ]
\eeq
From the above, we get
\beq
\Delta^{(k+1)} \leq \Big( 1 - \frac{1}{n} +\gamma \beta + 2 \gamma^2 \Lip{\bss}^2 \Big) \Delta^{(k)} + \Big( 2 \gamma^2 +\frac{\gamma}{\beta} \Big) \EE[  \| \hs{k} - \os^{(k)} \|^2 ]
\eeq
Setting $\overline{L}_{\sf v} = \max\{ \Lip{\bss} , \Lip{V} \}$, $\gamma = \frac{1}{\alpha c_1 \overline{L}_{\sf v} n^{2/3}}$, $\beta = \frac{c_1 \overline{L}_{\sf v}}{n^{1/3}}$, $(\alpha-1) c_1 \geq 4$, $\alpha \geq 6$, it is easy to check that
\beq
1 - \frac{1}{n} +\gamma \beta + 2 \gamma^2 \Lip{\bss}^2 \geq 1 - \frac{1}{n}
\eeq
and
\beq
1 - \frac{1}{n} +\gamma \beta + 2 \gamma^2 \Lip{\bss}^2
\leq 1 - \frac{1}{n} + \frac{1}{\alpha n} + \frac{ 2 }{ \alpha^2 c_1^2 n^{\frac{4}{3}} } \leq 1 - \frac{\alpha c_1 - c_1 - 2}{\alpha c_1 n} \leq 1 - \frac{2}{\alpha c_1 n}
\eeq
which shows that $1 - \frac{1}{n} +\gamma \beta + 2 \gamma^2 \Lip{\bss}^2 \in (0,1)$.
Observe that as $\Delta^{(0)} = 0$ and by telescoping, we have
\beq
\Delta^{(k+1)} \leq \Big( 2 \gamma^2 +\frac{\gamma}{\beta} \Big) \sum_{ \ell = 0 }^k \Big( 1 - \frac{1}{n} +\gamma \beta + 2 \gamma^2 \Lip{\bss}^2 \Big)^{k-\ell}  \EE[  \| \hs{\ell} - \os^{(\ell)} \|^2 ]
\eeq
Let $K_{\sf max} \in \NN$. Summing $k=0$ to $k=K_{\sf max}-1$ gives
\beq
\begin{split}
\sum_{k=0}^{K_{\sf max}-1} \Delta^{(k+1)} & \leq \Big( 2 \gamma^2 +\frac{\gamma}{\beta} \Big) \sum_{k=0}^{K_{\sf max}-1} \sum_{ \ell = 0 }^k \Big( 1 - \frac{1}{n} +\gamma \beta + 2 \gamma^2 \Lip{\bss}^2 \Big)^{k-\ell}  \EE[  \| \hs{\ell} - \os^{(\ell)} \|^2 ] \\
& = \Big( 2 \gamma^2 +\frac{\gamma}{\beta} \Big) \sum_{k=0}^{K_{\sf max}-1} \sum_{ \ell = 0 }^k \Big( 1 - \frac{1}{n} +\gamma \beta + 2 \gamma^2 \Lip{\bss}^2 \Big)^{\ell}  \EE[  \| \hs{k} - \os^{(k)} \|^2 ] \\
& \leq \frac{ 2 \gamma^2 +\frac{\gamma}{\beta} }{\frac{1}{n} - \gamma \beta - 2 \gamma^2 L_{\bss}^2} \sum_{k=0}^{K_{\sf max}-1}  \EE[  \| \hs{k} - \os^{(k)} \|^2 ].
\end{split}
\eeq
Summing up the both sides of \eqref{eq:lips_con} from $k=0$ to $k=K_{\sf max}-1$ yields
\beq
\begin{split}
& \EE \big[ V(\hat{\bss}^{(K_{\sf max})}) - V(\hat{\bss}^{(0)} ) \big] \\
& \leq  \sum_{k=0}^{K_{\sf max}-1} \Big\{ \Big( - \frac{\gamma}{c_1} + \gamma^2 \Lip{V} \Big) \EE[ \| \hs{k} - \os^{(k)} \|^2 ] + \gamma^2 \Lip{V} \Lip{\bss}^2 \Delta^{(k)} \Big\} \\
& \leq \sum_{k=0}^{K_{\sf max}-1} \Big\{ \Big( - \frac{\gamma}{c_1} + \gamma^2 \Lip{V} + \frac{ \big( \gamma^2 \Lip{V} \Lip{\bss}^2 \big) \big( 2 \gamma^2 +\frac{\gamma}{\beta} \big) }{\frac{1}{n} - \gamma \beta - 2 \gamma^2 L_{\bss}^2} \Big) \EE[ \| \hs{k} - \os^{(k)} \|^2 ]  \Big\}
\end{split}
\eeq
Furthermore,
\beq
\begin{split}
& \gamma^2 \Lip{V} + \frac{ \big( \gamma^2 \Lip{V} \Lip{\bss}^2 \big) \big( 2 \gamma^2 +\frac{\gamma}{\beta} \big) }{\frac{1}{n} - \gamma \beta - 2 \gamma^2 L_{\bss}^2} \\
& \overset{(a)}{\leq} \frac{1}{\alpha^2 c_1^2 \overline{L}_{\sf v} n^{4/3}} + \frac{\overline{L}_{\sf v} (\alpha^2 c_1^2  n^{4/3})^{-1} \big( \frac{2}{\alpha^2 c_1^2 \overline{L}_{\sf v}^2 n^{4/3}} + \frac{1}{\alpha c_1^2 \overline{L}_{\sf v}^2 n^{1/3}} \big) }{\frac{1}{n} - \frac{1}{\alpha n} - \frac{2}{\alpha^2 c_1^2 n^{4/3}} }\\
& = \frac{1}{\alpha^2 c_1^2 \overline{L}_{\sf v} n^{4/3}} + \frac{  \overline{L}_{\sf v} \big(\frac{2}{\alpha^2 c_1^2 \overline{L}_{\sf v}^2 n^{4/3}} + \frac{1}{\alpha c_1^2 \overline{L}_{\sf v}^2 n^{1/3}} \big) }{ (\alpha c_1  n^{1/3}) (\alpha-1) c_1 - 2 } \\
& \overset{(b)}{\leq} \frac{1}{\alpha^2 c_1^2 \overline{L}_{\sf v} n^{4/3}} + \frac{ \frac{1}{\alpha c_1^2 \overline{L}_{\sf v} n^{1/3}} \big(\frac{2}{\alpha n}  +1 \big) }{ 4 (\alpha c_1  n^{1/3}) - 2 }
\overset{(c)}{\leq} \frac{1}{\alpha^2 c_1^2 \overline{L}_{\sf v} n^{4/3}} + \frac{ 2 }{ 3 \alpha^2 c_1^3 \overline{L}_{\sf v} n^{2/3} } \\
& \leq  \frac{ 5/6 }{\alpha c_1^2 \overline{L}_{\sf v} n^{2/3}}
\end{split}
\eeq
where (a) uses $\overline{L}_{\sf v} \geq \max\{ \Lip{\bss}, \Lip{V} \}$, (b) is due to $(\alpha-1)c_1 \geq 4$ and (c) uses $\alpha c_1 n^{1/3} \geq 1$. Now, using the fact that $\frac{\gamma}{c_1} = \frac{1}{\alpha c_1^2 \overline{L}_{\sf v} n^{\frac{2}{3}}}$ and the lower bound $ \| \hs{k} - \os^{(k)} \|^2 \geq d_2^{-1} \| \grd V(\hs{k}) \|^2$, we have
\beq
\begin{split}
\EE \big[ V(\hat{\bss}^{(K_{\sf max})}) - V(\hat{\bss}^{(0)} ) \big]
& \leq - \frac{1}{6 \alpha c_1^2 \overline{L}_{\sf v} n^{\frac{2}{3}} } \sum_{k=0}^{K_{\sf max}-1} \EE[ \| \hs{k} - \os^{(k)} \|^2 ] \\
& \leq - \frac{1}{6 \alpha d_1^2 c_1^2 \overline{L}_{\sf v} n^{\frac{2}{3}} } \sum_{k=0}^{K_{\sf max}-1} \EE[ \| \grd V(\hs{k}) \|^2 ]
\end{split}
\eeq
Recalling that $K$ is an independent discrete r.v.~drawn uniformly from $\{1,\dots,K_{\sf max}\}$ and noting that $\alpha \geq 6$, we have
\beq
\begin{split}
\EE[ \| \grd V( \hs{K} ) \|^2 ] & = \frac{1}{K_{\sf max}} \sum_{k=0}^{K_{\sf max}-1}  \EE[ \| \grd V( \hs{k} ) \|^2 ] \leq n^{\frac{2}{3}} \!~ \frac{d_1^2 \overline{L}_{\sf v} (\alpha c_1)^2 (\EE \big[  V(\hat{\bss}^{(0)} ) -V(\hat{\bss}^{(K_{\sf max})})  \big]}{K_{\sf max}}
 \end{split}
\eeq

\section{Practical Applications of Stochastic EM methods}
This section provides implementation details and verify the model assumptions for the application examples provided.
Only in this section, for any $M \geq 2$, we denote
\beq \textstyle \label{eq:prob_simplex}
\Delta^M \eqdef \{ \omega_m \in \rset,~m=1,...,M-1 : \omega_m \geq 0,~\sum_{m=1}^{M-1} \omega_m \leq 1\} \subseteq \rset^{M-1}
\eeq
as the shorthand notation of the dimension reduced $M$-D probability simplex.

\subsection{Gaussian mixture models}\label{app:gmm}
\subsubsection{Model assumptions}
We first recognize that the constraint set for $\param$ is given by
\beq \textstyle
\Param = \Delta^M \times \rset^M.
\eeq
Using the partition of the sufficient statistics as
$S( y_i,z_i ) = ( S^{(1)}( y_i,z_i)^\top , S^{(2)}( y_i,z_i )^\top, S^{(3)}(y_i,z_i) )^\top  \in \rset^{M-1} \times \rset^{M-1} \times \rset$, the partition $\phi( \param ) = ( \phi^{(1)}( \param )^\top ,\phi^{(2)}( \param )^\top,\phi^{(3)}( \param ) )^\top \in \rset^{M-1} \times \rset^{M-1} \times \rset$ and the fact that $\indiacc{M}(z_i)= 1 - \sum_{m=1}^{M-1} \indiacc{m}(z_i)$, the complete data log-likelihood can be expressed as in \eqref{eq:exp} with
\beq \label{eq:gmm_exp}
\begin{split}
& s_{i,m}^{(1)} = \indiacc{m}(z_i), \quad \phi_m^{(1)}(\param) =   \left\{\log(\omega_m) -\frac{\mu_m^2}{2}\right\} - \left\{\log(1 - {\textstyle  \sum_{j=1}^{M-1}} \omega_j) - \frac{\mu_M^2}{2}\right\} \eqsp,\\
& s_{i,m}^{(2)} =   \indiacc{m}(z_i) y_i, \quad \phi^{(2)}_m(\param) =  {\mu_m} \eqsp, \quad s_i^{(3)} = y_i, \quad \phi^{(3)}(\param) = \mu_M \eqsp,
\end{split}
\eeq
and $\psi(\param) =   - \left\{\log(1 - \sum_{m=1}^{M-1} \omega_m) - \frac{\mu_M^2}{2 \sigma^2}\right\}$.
We also define for each $m \in \llbracket 1, M\rrbracket$,  $j \in \llbracket 1, 3 \rrbracket$, $s_{m}^{(j)} = n^{-1}\sum_{i=1}^n s_{i,m}^{(j)}$. Consider the following conditional expected value:
\beq \label{eq:cexp}
\widetilde{\omega}_m ( y_{i} ; \param ) \eqdef \EE_{\param}[ \mathbbm{1}_{\{z_i=m\}} | y= y_{i} ]
= \frac{ {\omega}_{m} \!~ {\rm exp}(-\frac{1}{2}( y_{i} - {\mu}_{i} )^2) }{  \sum_{j=1}^{M}{ {\omega}_{j} \!~ \exp(-\frac{1}{2}( y_{i} - {\mu}_{j} )^2)} } \eqsp,
\eeq
where $m \in \llbracket1,M\rrbracket$, $i \in \inter$ and $\param = ({\bm w}, {\bm{\mu}}) \in \Theta$.
In particular, given $\param \in \Param$, the {\sf E-step} updates in \eqref{eq:estep_upd} can be written as
\beq\label{eq:condexp_gmm}
\overline{\bss}_i( \param ) = \big( \underbrace{\widetilde{\omega}_1 ( y_{i} ; \param ),..., \widetilde{\omega}_{M-1} ( y_{i} ; \param )}_{\eqdef \overline{\bss}_i^{(1)}( \param )^\top} , \underbrace{y_i \widetilde{\omega}_1 ( y_{i} ; \param ), ..., y_i \widetilde{\omega}_M ( y_{i} ; \param )}_{\eqdef \overline{\bss}_i^{(2)}( \param )^\top}, \underbrace{y_i}_{\eqdef \overline{\bss}_i^{(3)}( \param )} \big)^\top.
\eeq

Recall that we have used the following regularizer:
\beq \textstyle \label{eq:regu}
\Pen( \param ) = \frac{\delta}{2} \sum_{m=1}^M \mu_m^2 - \epsilon \sum_{m=1}^M  \log ( \omega_m )  - \epsilon \log \big( 1 - \sum_{m=1}^{M-1} \omega_m \big) \eqsp,
\eeq
It can be shown that the regularized {\sf M-step} in \eqref{eq:mstep} evaluates to
\beq \label{eq:mstep_gmm}
\overline{\param} ( {\bm s} )
= \left(
\begin{array}{c}
( 1+\epsilon M )^{-1} \big( {s}_1^{(1)} + \epsilon, \dots,  {s}_{M-1}^{(1)} + \epsilon \big)^\top \vspace{.2cm}\\
 \big( ({s}_1^{(1)} + \delta )^{-1} {s}_1^{(2)}  , \dots, ({s}_{M-1}^{(1)} + \delta )^{-1} {s}_{M-1}^{(2)}  \big)^\top \vspace{.2cm} \\
  \big(1 - \sum_{m=1}^{M-1}s_m^{(1)} +  \delta\big)^{-1} \big( s^{(3)} - \sum_{m=1}^{M-1} s_m^{(2)} \big)
\end{array}
\right)
= \left(
\begin{array}{c}
\overline{\bm{\omega}} ( {\bm s}) \\
\overline{\bm{\mu}} ( {\bm s}) \\
\overline{\mu}_M ( {\bm s})
\end{array}
\right) \eqsp.
\eeq
where we have defined for all $m \in \llbracket1,M\rrbracket$ and $j \in \llbracket1,3\rrbracket$ , $ {s}_m^{(j)}  = n^{-1} \sum\nolimits_{i=1}^n s_{i,m}^{(j)}$.

To analyze the convergence of the EM methods, we verify H\ref{ass:compact} to H\ref{ass:eigen} for the GMM example as follows.

To verify H\ref{ass:compact}, we observe that the set $\Zset$ is the compact interval $\inter[M]$, in addition, the sufficient statistics defined in \eqref{eq:gmm_exp} also leads to a bounded and closed $\Sset$.

To verify H\ref{ass:regularity-phi-psi}, we observe that
the Jacobian matrix $\jacob{\phi}{\param}{\param}$ 
can be computed as
\beq\label{eq:jacob}
\jacob{\phi}{\param}{\param} = \left(
\begin{array}{ccc}
\frac{1}{1 - \sum_{m=1}^{M-1} \omega_m} {\bf 1}{\bf 1}^\top + {\rm Diag}( \frac{\bf 1}{\bm{\omega}} ) & - {\rm Diag}( \bm{\mu} ) & \mu_M {\bf 1} \\
{\bm 0} & {\bm I} & {\bm 0} \\
{\bm 0} & {\bm 0} & 1
\end{array}
\right),
\eeq
where we have denoted $ \frac{\bf 1}{\bm{\omega}} $ as the $(M-1)$-dimensional vector $\big( \frac{\bf 1}{\omega_{1}} , \ldots,  \frac{\bf 1}{\omega_{M-1}} \big)$. 
We observe that it is a bounded matrix and it is smooth \wrt $\param$.

We verify H\ref{ass:expected} next, \ie the Lipschitz continuity of $p( z_i | y_i; \param )$, w.r.t to $\param$ noting that for all $i \in \inter[n]$ and $ m \in \inter[M]$, $p( z_i = m | y_i; \param ) =\EE_{\param}[ \mathbbm{1}_{\{z_i=m\}} | y= y_{i} ] =\widetilde{\omega}_m ( y_{i} ;\param )$. Observe that $p( z_i = m | y_i; \param )$ is given by the softmax function and the desired Lipschitz property follows.

Next, we observe that with the designed penalty, the function $\param \mapsto L(\bss,\param)$ admits a unique global minima with $\op(\bss) \in {\rm int}( \Param)$ for all $\bss \in \Sset$. Second, since $\op(\bss) \in {\rm int}( \Param)$, the Jacobian matrix defined in \eqref{eq:jacob} must be full rank. Lastly, the $L_{\theta}$-Lipschitzness of $\op(\bss)$ can be deduced by inspecting \eqref{eq:mstep_gmm}.  
The above show that Assumption H\ref{ass:reg} is verified.

Finally, we calculate the quantity $\operatorname{B}( \bss )$ defined in \eqref{eq:Bss}. Observe that the Hessian $\hess{L}{\param}(\bss,\param)$ is:
\beq
\hess{L}{\param}(\bss,\param) =  \left(
\begin{array}{ccc}
\frac{1 + \epsilon - \sum_{m=1}^{M-1} s_m^{(1)} }{(1 - \sum_{m=1}^{M-1} \omega_m)^2} {\bf 1}{\bf 1}^\top + {\rm Diag}( \frac{ {\bm s}^{(1)} + \epsilon {\bf 1}}{\bm{\omega}^2} ) & {\bm 0} & {\bm 0} \\
{\bm 0} & {\rm Diag}( {\bm s}^{(1)} + \delta {\bf 1}  ) & {\bm 0} \\
{\bm 0} & {\bm 0} & \delta + 1  - \sum_{m=1}^{M-1} s_m^{(1)}
\end{array}
\right)
\eeq
We can rewrite $\operatorname{B}( \bss )$ as an outer product:
\beq
\operatorname{B}( \bss )  \eqdef\jacob{ \phi }{ \param }{ \mstep{\bss} } \Big( \hess{L}{\param}( {\bss},  \mstep{\bss} )  \Big)^{-1} \jacob{ \phi }{ \param }{ \mstep{\bss} }^\top =  \bm{\mathcal{J}} ( {\bm s} ) \bm{\mathcal{J}} ( {\bm s} )^\top
\eeq
where
\beq\label{eq:Jdef}
\bm{\mathcal{J}} ( {\bm s}  ) \eqdef \jacob{\phi}{\param}{ \mstep{\bm s} } \left(
\begin{array}{ccc}
{\bm H}_{11}^{-\frac{1}{2}} & {\bm 0} & {\bm 0} \\
{\bm 0} & {\rm Diag}( \frac{ \bm 1}{ \sqrt{ {\bm s}^{(1)} + \delta {\bf 1}  }} ) & {\bm 0} \\
{\bm 0} & {\bm 0} & \frac{1}{\sqrt{\delta + 1 - \sum_{m=1}^{M-1} s_m^{(1)}}}
\end{array}
\right)
\eeq
and
\beq
{\bm H}_{11} \eqdef
\frac{1 + \epsilon - \sum_{m=1}^{M-1} s_m^{(1)} }{(1 -  \frac{ {\bf 1}^\top ( {\bm s}^{(1)} + \epsilon {\bf 1} ) }{1+ \epsilon M} )^2} {\bf 1}{\bf 1}^\top + {\rm Diag}( \frac{ (1+ \epsilon M)^2 }{ {\bm s}^{(1)} + \epsilon {\bf 1} } ).
\eeq
Note that $\bm{\mathcal{J}} ( {\bm s}  )$ is a bounded and full rank matrix which yields to the upper and lower bounds on eigenvealues in H\ref{ass:eigen}. From \eqref{eq:Jdef}, we note that $\operatorname{B}( \bss ) =\bm{\mathcal{J}} ( {\bm s}  )\bm{\mathcal{J}} ( {\bm s}  )^\top $ is Lipschitz continuous, \ie, there exists a constant $\Lip{B}$ such that for all $\bss, \bss' \in \Sset^2$, we have $ \| \operatorname{B}( \bss ) - \operatorname{B}( \bss' )  \| \leq \Lip{B} \| {\bss} - {\bss}' \|$.

\subsubsection{Algorithms updates}
In the sequel, for all $i \in \inter[n]$ and iteration $k$, the conditional expectation $\overline{\bss}_{i}^{(k)}$ is defined by \eqref{eq:condexp_gmm} and is equal to:
\beq
\overline{\bss}_{i}^{(k)}
= \left(
\begin{array}{c}
 \big( \widetilde{\omega}_1 ( y_{i}; \hp{k} ), \dots, \widetilde{\omega}_{M-1} ( y_{i} ; \hp{k} ) \big)^\top \\
  \big( y_{i} \widetilde{\omega}_1 ( y_{i}; \hp{k} ), \dots, y_{i}  \widetilde{\omega}_{M-1} ( y_{i} ; \hp{k} ) \big)^\top \\
  y_{i}
\end{array}
\right) \eqsp.
\eeq
At iteration $k$, the several E-steps defined by \eqref{eq:iem} or \eqref{eq:oem} or \eqref{eq:svrgem} or \eqref{eq:sagaem} leads to the definition of the quantity $\hat{\bss}^{(k+1)} $. For the GMM example, after the initialization of the quantity $\hat{\bss}^{(0)} = n^{-1} \sum\nolimits_{i=1}^n \overline{\bss}_i^{(0)} $, those E-steps break down as follows:

\textbf{Batch EM (EM):} for all $i \in \inter$, compute $\overline{\bss}_{i}^{(k)}$ and set 
\beq
\hat{\bss}^{(k+1)} = n^{-1} \sum\nolimits_{i=1}^n \overline{\bss}_i^{(k)} \eqsp.
\eeq

\textbf{Online EM (\SEM):} draw an index $i_k$ uniformly at random on $\inter[n]$, compute $\overline{\bss}_{i_k}^{(k)}$ and set 
\beq
\hat{\bss}^{(k+1)} = (1-\gamma_k)\hat{\bss}^{(k)}  + \gamma_k \overline{\bss}_{i_k}^{(k)} \eqsp.
\eeq

\textbf{Incremental EM (\IEM):} draw an index $i_k$ uniformly at random on $\inter[n]$, compute $\overline{\bss}_{i_k}^{(k)}$ and set 
\beq
\hat{\bss}^{(k+1)} = \hat{\bss}^{(k)}  +  \overline{\bss}_{i_k}^{(k)} -  \overline{\bss}_{i_k}^{(\tau_i^k)} =   n^{-1} \sum\nolimits_{i=1}^n \overline{\bss}_i^{(\tau_i^k)} \eqsp.
\eeq

\textbf{Variance reduced stochastic EM (\SEMVR):} draw an index $i_k$ uniformly at random on $\inter[n]$, compute $\overline{\bss}_{i_k}^{(k)}$
and set 
\beq
\hat{\bss}^{(k+1)} =(1-\gamma) \hat{\bss}^{(k)}  + \gamma \Big( \overline{\bss}_{i_k}^{(k)} -  \overline{\bss}_{i_k}^{(\ell(k))}   + \overline{\bss}^{(\ell(k))}  \Big)
\eeq
where $ \overline{\bss}_{i_k}^{(\ell(k))}$ and $\overline{\bss}^{(\ell(k))}$ were computed at iteration $\ell(k)$, defined as the first iteration number in the epoch that iteration $k$ is in.

\textbf{Fast Incremental EM (\FIEM):} draw two different and independent indices $(i_k, j_k)$ uniformly at random on $\inter[n]$, compute the quantities $\overline{\bss}_{i_k}^{(k)}$ and $\overline{\bss}_{j_k}^{(k)}$ and set 
\beq
\begin{split}
& \hat{\bss}^{(k+1)} =(1-\gamma) \hat{\bss}^{(k)}  + \gamma \Big( \overline{\StocEstep}^{(k)}+ \overline{\bss}_{i_k}^{(k)} - \overline{\bss}_{i_k}^{(t_{i_k}^{k})} \Big)\\
& \overline{\StocEstep}^{(k+1)} = \overline{\StocEstep}^{(k)} + n^{-1}
\big( \overline{\bss}_{j_k}^{(k)}  - \overline{\bss}_{j_k}^{(t_{j_k}^k)} \big)\\
\end{split}
\eeq

Finally, the $k$-th update reads $\hp{k+1} = \overline{\param} (\hat{\bss}^{(k+1)})$ where the function ${\bm s} \to \overline{\param}({\bm s})$ is defined by \eqref{eq:mstep_gmm}.

\subsection{Probabilistic Latent Semantic Analysis}\label{app:plsa}

\subsubsection{Model assumptions}
The constraint set $\Param$ is given by
\beq \label{eq:param_app}
\Param = \left( \times_{d \in \inter[D]} \Delta^K \right) \times \left( \times_{k \in \inter[K]} \Delta^V \right).
\eeq
For the {\sf sE-step} \eqref{eq:sestep} in the EM methods, 
we compute the expected complete data statistics as
\beq\label{eq:condexp_plsa}
\begin{split}
& \textstyle \bssdoc_{i,d,k}(\pardoc,\partop) = \indiacc{d}(\obsdoc_i) \Big( \sum_{\ell =1}^K \pardoc_{d,\ell} \partop_{\ell, \obsword_i} \Big)^{-1} \pardoc_{d,k} \partop_{k, \obsword_i}, \\
& \textstyle \bsstop_{i,k,v}(\pardoc,\partop) =  \indiacc{v}(\obsword_i) \Big( \sum_{\ell =1}^K \pardoc_{\obsdoc_i,\ell} \partop_{\ell, v}  \Big)^{-1} \pardoc_{\obsdoc_i,k} \partop_{k, v},
\end{split}
\eeq
for  each $(i, k, d, v) \in \inter[n] \times \inter[K] \times \inter[D]\times \inter[V]$.
Meanwhile, the regularized {\sf M-step} \eqref{eq:mstep} in the EM methods evaluates to:
\beq \label{eq:mstep_plsa}
\left(
\begin{array}{c}
\opardoc_{d,k} ( {\bm s}) \\
\opartop_{k,v} ( {\bm s})
\end{array}
\right)
= \left(
\begin{array}{c}
\Big( \sum_{i=1}^n \sum_{k'=1}^K {\bm s}^{(t|d)}_{i,d,k'}+\alpha' K \Big)^{-1} \big( \sum_{i=1}^n {\bm s}^{(t|d)}_{i,d,k} +\alpha'\big) \vspace{.1cm}\\
\Big(\sum_{i=1}^n \sum_{\ell=1}^V {\bm s}^{(w|t)}_{i,k,\ell}+  \beta' V \Big)^{-1} \big( \sum_{i=1}^n {\bm s}^{(w|t)}_{i,k,v}+\beta' \big) \vspace{.2cm}
\end{array}
\right),
\eeq
for each $(k,d, v) \in \inter[K]\times\inter[D]\times\inter[V]$.

Using the partition of the sufficient statistics as
$S(  y_i,z_i) = ( S^{(t|d)}( y_i,z_i)^\top , S^{(w|t)}( y_i,z_i)^\top)^\top \in  \rset^{K  D + K  V} $, the partition $\phi( \param ) = ( \phi^{(t|d)}( \param )^\top ,\phi^{(w|t)}( \param )^\top )^\top \in  \rset^{K  D + K V} $, the complete log-likelihood \eqref{eq:comp_like_plsa} can be expressed in the standard form as \eqref{eq:exp} with
\beq \label{eq:plsa_exp}
\begin{split}
& {\bm s}^{(t|d)}_{i,d,k} =  \indiacc{k,d}(z_i,\obsdoc_i) , \quad
   \phi_{d,k}^{(t|d)}(\param) =   \log(\pardoc_{d,k}) \eqsp,\\
  & {\bm s}^{(w|t)}_{i,k,v}= \indiacc{k,v}(z_i,\obsword_i) , \quad \phi^{(w|t)}_{k,v}(\param) = \log (\partop_{k,v}) \eqsp,
\end{split}
\eeq
Assumption  H\ref{ass:compact} is verified with $\Zset = \inter[K]$ and the sufficient statistics defined in \eqref{eq:plsa_exp} that leads to a compact $\Sset$.

By using the vectorization of $\param$ as an $(K-1)D + (V-1)K$-dimensional vector, 
we can calculate the Jacobian as follows. 
In particular,
\beq
\jacob{\pardoc_{d,k}}{\phi_{d',k'}^{(t|d)}}{\param} = \begin{cases}
0 & \text{if}~d' \neq d, \\
\frac{1}{1 - \sum_{\ell=1}^{K-1} \pardoc_{d,\ell}} &  \text{if}~d' = d, k' \neq k \\
\frac{1}{\pardoc_{d,k}} &  \text{if}~d' = d, k' = k.
\end{cases},
\jacob{\partop_{k,v}}{\phi_{k',v'}^{(w|t)}}{\param} = \begin{cases}
0 & \text{if}~k' \neq k, \\
\frac{1}{1 - \sum_{\ell=1}^{V-1} \partop_{k,\ell}} &  \text{if}~k' = k, v' \neq v \\
\frac{1}{\partop_{k,v}} &  \text{if}~k' = k, v' = v.
\end{cases}
\eeq
With the above definitions,
it can be verified that the Jacobian matrix is full rank and smooth \wrt $\param$ for any $\param \in {\rm int}(\Param)$.
This confirms H\ref{ass:regularity-phi-psi}.
%

Next, we verify H\ref{ass:expected}, \ie the Lipschitz continuity of $p( z_i | y_i; \param )$, w.r.t to $\param$. Note that for all $(i,k,d) \in \inter[n] \times \inter[K] \times \inter[D]$, $p( z_i = k | y_i; \pardoc_{d,k},\partop_{k,v} ) =\EE_{\param}[ \indiacc{k,d}(z_i,\obsdoc_i) | y_{i} ] =\bssdoc_{i,k,d}(\pardoc,\partop)$ as defined in \eqref{eq:condexp_plsa}. Observe that as we focus on $\param \in {\rm int}( \Param )$, each of $\pardoc_{d,\ell} \partop_{\ell, \obsword_i}$, $\pardoc_{\obsdoc_i,\ell} \partop_{\ell, v}$ is strictly positive and strictly less than one. The Lipschitz property follows from the expression \eqref{eq:condexp_plsa}.

The expression of the regularized complete log-likelihood, $\param \to L(s,\param)$, is defined as:
\beq \notag
L(s,\param) = - \sum_{k=1}^K\sum_{d=1}^D  {\bm s}^{(t|d)}_{i,k,d} \log(\pardoc_{d,k}) - \alpha' \log(\pardoc_{d,k}) -  \sum_{k=1}^K \sum_{v=1}^V {\bm s}^{(w|t)}_{i,k,v} \log (\partop_{k,v}) - \beta' \log (\partop_{k,v})  \eqsp,
\eeq
This function admits a unique minimum in ${\rm int}(\Param)$ from the strict concavity of the logarithm, as the regularizations are active with $\alpha' , \beta' > 0$.
By the same virtue of the verification of H\ref{ass:regularity-phi-psi}, we observe that  H\ref{ass:reg} 
can be satisfied.

We first calculate the quantity $\operatorname{B}( \bss )$ defined in \eqref{eq:Bss}.
Using the vectorization of $\param$ as a $(K-1)D + (V-1)K$-dimensional vector, we observe that the Hessian of the function $\param \mapsto L(\bss,\param)$ \wrt to $\param$ has a block diagonal structure with $D+K$ blocks --- the $d$th block which corresponds to $\pardoc_{d,\cdot}$ is given by
\beq
\big[ \hess{L}{\param}(\bss,\param) \big]_d =  
\frac{{\bm s}^{(t|d)}_{K,d}  + \alpha'}{(1- \sum_{k=1}^{K-1} \pardoc_{d,k})^2} {\bf 1}{\bf 1}^\top + {\rm Diag}( \frac{ {\bm s}^{(t|d)} + \alpha' {\bf 1}}{(\pardoc)^2} )
\eeq
while the $(D+k)$th block which corresponds to $\partop_{k,\cdot}$ is given by 
\beq 
\big[ \hess{L}{\param}(\bss,\param) \big]_{D+k} = 
\frac{{\bm s}^{(w|t)}_{k,V}  + \beta'}{(1- \sum_{\ell=1}^{V-1} \partop_{k,\ell})^2} {\bf 1}{\bf 1}^\top + {\rm Diag}( \frac{ {\bm s}^{(w|t)} + \beta' {\bf 1}}{(\partop)^2} )
\eeq
Since each block in the above Hessian matrix is positive definite,  the matrix 
\beq
\operatorname{B}( \bss )  \eqdef\jacob{ \phi }{ \param }{ \mstep{\bss} } \Big( \hess{L}{\param}( {\bss},  \mstep{\bss} )  \Big)^{-1} \jacob{ \phi }{ \param }{ \mstep{\bss} }^\top =  \bm{\mathcal{J}} ( {\bm s} ) \bm{\mathcal{J}} ( {\bm s} )^\top
\eeq
 is positive definite and bounded. Furthermore, there exists a constant $\Lip{B} $ such that
$ \| \operatorname{B}( \bss ) - \operatorname{B}( \bss' )  \| \leq \Lip{B} \| {\bss} - {\bss}' \|$.
Finally, this confirms H\ref{ass:eigen}.

\subsubsection{Algorithms updates}
In the sequel, for all $(i, d, k, v) \in \inter[n] \times \inter[D] \times \inter[K] \times \inter[V]$ the conditional expectations $\bssdoc_{i,k,d}(\estpardoc{\iter},\estpartop{\iter})$ and $\bsstop_{i,k,v}(\estpardoc{\iter},\estpartop{\iter})$ are defined by \eqref{eq:condexp_plsa}.
For the pLSA example, after the initialization of the quantity $\big({\bm s}^{(1)}_{k,d}\big)^{0} = n^{-1} \sum\nolimits_{i=1}^n  \bssdoc_{i,k,d}(\estpardoc{0},\estpartop{0}) $ and $\big({\bm s}^{(2)}_{k,v}\big)^{0} = n^{-1} \sum\nolimits_{i=1}^n  \bsstop_{i,k,v}(\estpardoc{0},\estpartop{0}) $, the several E-steps break down as follows:

\textbf{Batch EM (EM):} At iteration $\iter$: update the statistics for all $(d, k, v) \in \inter[D] \times \inter[K] \times \inter[V]$ :
\beq
 \big({\bm s}^{(1)}_{k,d}\big)^{\iter+1} = \sum_{i=1}^{n}  \bssdoc_{i,k,d}(\estpardoc{\iter},\estpartop{\iter})  \quad \textrm{and} \quad \big({\bm s}^{(2)}_{k,v}\big)^{\iter+1} =  \sum_{i=1}^{n} \bsstop_{i,k,v}(\estpardoc{\iter},\estpartop{\iter})
\eeq

\textbf{Online EM (\SEM):} At iteration $\iter$, update the statistics for all $(d, k, v) \in \inter[D] \times \inter[K] \times \inter[V]$ :
\beq
\begin{split}
&  \big({\bm s}^{(1)}_{k,d}\big)^{\iter+1} = (1- \gamma_{\iter}) \big({\bm s}^{(1)}_{k,d}\big)^{\iter} +\gamma_{\iter} \bssdoc_{i_{\iter},k,d}(\estpardoc{\iter},\estpartop{\iter})\\
&   \big({\bm s}^{(2)}_{k,v}\big)^{\iter+1}   = (1-\gamma_{\iter})  \big({\bm s}^{(2)}_{k,v}\big)^{\iter} + \gamma_{\iter}  \bsstop_{i_{\iter},k,v}(\estpardoc{\iter},\estpartop{\iter}) \\
\end{split}
\eeq

\textbf{Incremental EM (\IEM):} At iteration $\iter$, update the statistics for all $(d, k, v) \in \inter[D] \times \inter[K] \times \inter[V]$ :
\beq
\begin{split}
&  \big({\bm s}^{(1)}_{k,d}\big)^{\iter+1}  = \big({\bm s}^{(1)}_{k,d}\big)^{\iter} + \bssdoc_{i_{\iter},k,d}(\estpardoc{\iter},\estpartop{\iter}) - \bssdoc_{i_{\iter},k,d}(\estpardoc{\tau_{i_{\iter}}^{\iter}},\estpartop{\tau_{i_{\iter}}^{\iter}})\\
&   \big({\bm s}^{(2)}_{k,v}\big)^{\iter+1}   =  \big({\bm s}^{(2)}_{k,v}\big)^{\iter} + \bsstop_{i_{\iter},k,v}(\estpardoc{\iter},\estpartop{\iter})  - \bsstop_{i_{\iter},k,v}(\estpardoc{\tau_{i_{\iter}}^{\iter}},\estpartop{\tau_{i_{\iter}}^{\iter}})\\
\end{split}
\eeq

\textbf{Variance reduced stochastic EM (\SEMVR):} At iteration $\iter$, draw an index $i_\iter$ and  update the statistics for all $(d, k, v) \in \inter[D] \times \inter[K] \times \inter[V]$ :
\beq
\begin{split}
&  \big({\bm s}^{(1)}_{k,d}\big)^{\iter+1}  =  (1- \gamma) \big({\bm s}^{(1)}_{k,d}\big)^{\iter} \\
& + \gamma \Big( \bssdoc_{i_{\iter},k,d}(\estpardoc{\iter},\estpartop{\iter}) - \bssdoc_{i_{\iter},k,d}(\estpardoc{(\ell(k))},\estpartop{(\ell(k))}) + \bssdoc(\estpardoc{(\ell(k))},\estpartop{(\ell(k))})\Big)\\
&   \big({\bm s}^{(2)}_{k,v}\big)^{\iter+1}   =   (1- \gamma) \big({\bm s}^{(2)}_{k,v}\big)^{\iter} \\
& + \gamma \Big( \bsstop_{i_{\iter},k,v}(\estpardoc{\iter},\estpartop{\iter}) - \bsstop_{i_{\iter},k,v}(\estpardoc{(\ell(k))},\estpartop{(\ell(k))}) + \bsstop(\estpardoc{(\ell(k))},\estpartop{(\ell(k))})\Big)\\
\end{split}
\eeq

\textbf{Fast Incremental EM (\FIEM):} At iteration $\iter$, draw two indices $(i_\iter, j_\iter)$ independently and update the statistics for all $(d, k, v) \in \inter[D] \times \inter[K] \times \inter[V]$ :
\begin{alignat*}{2}
&  \big({\bm s}^{(1)}_{k,d}\big)^{\iter+1}  = && (1- \gamma) \big({\bm s}^{(1)}_{k,d}\big)^{\iter} \\
& && + \gamma \Big( \bssdoc_{i_{\iter},k,d}(\estpardoc{\iter},\estpartop{\iter}) - \bssdoc_{i_{\iter},k,d}(\estpardoc{(t_{i_\iter}^\iter)},\estpartop{(t_{i_\iter}^\iter)})+ \big( \overline{\StocEstep}^{(1)}_{k,d}\big)^{\iter}  \Big) \\
&\big( \overline{\StocEstep}^{(1)}_{k,d}\big)^{\iter+1}  = && \big( \overline{\StocEstep}^{(1)}_{k,d}\big)^{\iter}  +n^{-1} \Big( \bssdoc_{j_{\iter},k,d}(\estpardoc{\iter},\estpartop{\iter}) - \bssdoc_{j_{\iter},k,d}(\estpardoc{(t_{j_\iter}^\iter)},\estpartop{(t_{j_\iter}^\iter)})\Big)\\
&   \big({\bm s}^{(2)}_{k,v}\big)^{\iter+1}   = &&   (1- \gamma) \big({\bm s}^{(2)}_{k,v}\big)^{\iter} \\
& && + \gamma \Big( \bssdoc_{i_{\iter},k,v}(\estpardoc{\iter},\estpartop{\iter}) - \bssdoc_{i_{\iter},k,v}(\estpardoc{(t_{i_\iter}^\iter)},\estpartop{(t_{i_\iter}^\iter)})+ \big( \overline{\StocEstep}^{(2)}_{k,v}\big)^{\iter}\Big)  \\
&\big( \overline{\StocEstep}^{(2)}_{k,v}\big)^{\iter+1}  = && \big( \overline{\StocEstep}^{(2)}_{k,v}\big)^{\iter} + \gamma n^{-1} \Big( \bssdoc_{j_{\iter},k,v}(\estpardoc{\iter},\estpartop{\iter}) - \bssdoc_{j_{\iter},k,v}(\estpardoc{(t_{j_\iter}^\iter)},\estpartop{(t_{j_\iter}^\iter)})\Big)
\end{alignat*}

Finally, at iteration $\iter$, for $(k,d, v) \in \inter[K]\times\inter[D]\times\inter[V]$, the {\sf M-step} in \eqref{eq:mstep} evaluates to:
\beq \label{eq:mstep_plsa}
\left(
\begin{array}{c}
\estpardoc{\iter+1}\\
\estpartop{\iter+1}
\end{array}
\right)
= \left(
\begin{array}{c}
\big(  \sum_{k'=1}^K \big({\bm s}^{(1)}_{k',d}\big)^{\iter+1} +\alpha' K \big)^{-1} \big( \big({\bm s}^{(1)}_{k,d}\big)^{\iter+1} +\alpha'\big) \vspace{.2cm}\\
\big(\sum_{\ell=1}^V \big({\bm s}^{(2)}_{k,\ell}\big)^{\iter+1} +  \beta' V \big)^{-1} \big( \big({\bm s}^{(2)}_{k,v}\big)^{\iter+1} +\beta' \big) \vspace{.2cm}
\end{array}
\right)\eqsp.
\eeq

\section{Local Linear Convergence of \FIEM}
In this section, we prove that the \FIEM~method converges locally at a linear rate to a stationary point, under a similar set of assumptions as in \citep{chen2018stochastic}. Note that some of the following assumptions can be difficult to verify, and our analysis here is merely a proof of concept. 

Consider a stationary point $\param^\star$ to problem \eqref{eq:em_motivate} and its corresponding sufficient statistics $\bss^\star$, also a stationary  point to \eqref{eq:em_sspace}. To simplify notations, we follow \citep{chen2018stochastic} and write the complete sufficient statistics as $F( \bss' ) \eqdef \overline{\bss} ( \mstep{ \bss' } )$, and also the $i$th sufficient statistics as $f_i( \bss' ) \eqdef \overline{\bss}_i ( \mstep{ \bss' })$. We assume the following:
\begin{assumptionB}\label{assb:hess}
The Hessian matrix $\grd^2 \overline{\cal L}( \param^\star )$ is strictly positive definite such that $\param^\star$ is a strict local minimizer of problem \eqref{eq:em_motivate}.
\end{assumptionB}
\begin{assumptionB} \label{assb:local}
For any $k \geq 1$, we have $\| \hat{\bss}^k - \bss^\star \| \leq \frac{\lambda}{\Lip{\bss}}$, where $\Lip{\bss}$ was defined in our Lemma~\ref{lem:smooth} and $1-\lambda$ is the maximum eigenvalue of the Jacobian matrix $\jacob{F}{\bss}{\bss^\star}$.
\end{assumptionB}
The above assumptions correspond to assumptions (a), (c) in \citep[Theorem 1]{chen2018stochastic}, 
while we note that assumption (b) therein are shown in our Lemma~\ref{lem:smooth}. 

We remark that B\ref{assb:hess} is strictly stronger than H\ref{ass:reg} used in our global convergence analysis. The latter makes assumption on the actual objective function $\overline{\cal L}( \param^\star )$ instead of the surrogate function $\param \rightarrow L( \bss, \param )$. Our proof goes as follows.

\begin{Prop}
Under Assumption~B\ref{assb:hess}, B\ref{assb:local} and the conditions such that our Lemma~\ref{lem:smooth} holds. The \FIEM~method converges linearly such that
\beq
\EE[ \| \hs{k+1} - \bss^\star \|^2 ] \leq (1 - \delta)^{k+1} \| \hs{0} - \bss^\star \|^2,~\forall~k \geq 0,
\eeq
where $\delta = \Theta(1/n)$ with an appropriately chosen step size $\gamma$.
\end{Prop}

\begin{proof} (Sketch)
For $k \in \nset^*$, denote by $\mcf_{k}$ the $\sigma$-algebra generated by the random variables $i_0, j_0,\dots,i_{k}, j_{k}$. Consider 
\beq
\begin{split}
& \EE \big[ \| \hs{k+1} - \bss^\star \|^2 | \mcf_{k} \big] = \EE  \big[ \| \hs{k} - \gamma ( \hs{k} - \StocEstep^{(k+1)} ) - \bss^\star \|^2 | \mcf_{k} \big] \\
& = \EE \big[ \| (1-\gamma) \hs{k} + \gamma F( \hs{k} ) - \bss^\star + \gamma \big( \StocEstep^{(k+1)} - F(\hs{k}) \big) \|^2 | \mcf_{k} \big]
\end{split}
\eeq
Note that as $\EE[ \StocEstep^{(k+1)} - F(\hs{k}) | \mcf_{k} ] = 0$, we have
\beq \label{eq:linear_0mean}
\begin{split}
& \EE \big[ \| \hs{k+1} - \bss^\star \|^2 | \mcf_{k} \big] \\
& = \EE \big[ \| (1-\gamma) \hs{k} + \gamma F( \hs{k} ) - \bss^\star \|^2 | \mcf_{k} \big] + \gamma^2 \EE \big[ \| \StocEstep^{(k+1)} - F(\hs{k}) \|^2 | \mcf_{k} \big]
\end{split}
\eeq
Repeating the analysis in (9) of \citep{chen2018stochastic}, we arrive at the upper bound
\beq
\EE \big[ \| (1-\gamma) \hs{k} + \gamma F( \hs{k} ) - \bss^\star \|^2 | \mcf_{k} \big] 
\leq ( 1 - \gamma \lambda / 2 ) \| \hs{k} - \bss^\star \|^2 
\eeq
On the other hand, applying \citep[Lemma 3]{defazio2014saga} shows that
\beq
\begin{split}
 \EE \big[ \| \StocEstep^{(k+1)} - F(\hs{k}) \|^2 | \mcf_{k} \big] 
& \leq 2 \Big( \| f_{i_k} ( \hs{\tau_{i_k}^k} ) - f_{i_k} ( \bss^\star ) \|^2 + \| f_{i_k} ( \hs{k}) - f_{i_k} ( \bss^\star ) \|^2  \Big) \\
& \leq 2 \Lip{\bss}^2 \Big( \| \hs{\tau_{i_k}^k}  - \bss^\star \|^2 + \| \hs{k} - \bss^\star \|^2 \Big) 
\end{split}
\eeq
Denote the total expectation as $h_k \eqdef \EE [ \| \hs{k} - \bss^\star \|^2 ]$, and
taking the total expectation on both sides   yields
\beq
 \EE \big[ \| \StocEstep^{(k+1)} - F(\hs{k}) \|^2 \big] \leq 
 2 \Lip{\bss}^2 \big( h_k + {\textstyle \frac{1}{n} \sum_{i=1}^n} h_{\tau_{i}^k} \big)
\eeq
Substituting the above into \eqref{eq:linear_0mean} yields
\beq \label{eq:before_recur}
h_{k+1} \leq \Big( 1 - \gamma \frac{\lambda}{2} + 2 \gamma^2 \Lip{\bss}^2 \Big) h_k + 2\gamma^2 \Lip{\bss}^2 \big( {\textstyle \frac{1}{n} \sum_{i=1}^n} h_{\tau_{i}^k} \big)
\eeq
Moreover, we observe the following recursion through evaluating the expectation
\beq
{\frac{1}{n} \sum_{i=1}^n} h_{\tau_{i}^k} = \frac{1}{n} h_{k-1} + \Big(1 - {\frac{1}{n}} \Big) {\frac{1}{n} \sum_{i=1}^n} h_{\tau_{i}^{k-1}} \leq \frac{1}{n} \sum_{\ell=0}^{k-1} 
\Big(1 - {\frac{1}{n}} \Big)^{k-\ell-1} h_\ell
\eeq 
Therefore, \eqref{eq:before_recur} simplifies to
\beq \label{eq:after_recur}
h_{k+1} \leq \Big( 1 - \gamma \frac{\lambda}{2} + 2 \gamma^2 \Lip{\bss}^2 \Big) h_k + \frac{2\gamma^2 \Lip{\bss}^2}{n} \sum_{\ell=0}^{k-1} 
\Big(1 - {\frac{1}{n}} \Big)^{k-\ell-1} h_\ell
\eeq
To this end, we let $a = \frac{\lambda}{2}, b = 2 \Lip{\bss}^2, c = 2 \Lip{\bss}^2$ and consider the following inequality, 
\beq
h_{k+1} \leq \big( 1 - \gamma a + \gamma^2  b \big) h_k + \frac{\gamma^2 c}{n} \sum_{\ell=0}^{k-1} \Big(1 - {\frac{1}{n}} \Big)^{k-\ell-1} h_\ell 
\eeq
We claim that for a sufficiently small step size $\gamma$, there exists  $\delta \in (0,1]$ such that $h_{k} \leq (1 - \delta)^k h_0$ for all $k$.
The proof can be achieved using induction. The base case is straightforward since:
\beq
h_1 \leq (1 - \gamma a + \gamma^2 b ) h_0 
\eeq
For the induction case, we assume that $h_{\tau} \leq (1-\delta)^{\tau} h_0$ for $\tau=1,2....,k$. We observe that the induction hypothesis implies
\beq
\begin{split}
\frac{h_{k+1}}{h_0} & \leq \big( 1 - \gamma a + \gamma^2  b \big) (1-\delta)^{k} + \frac{\gamma^2 c}{n} \sum_{\ell=0}^{k-1} \Big(1 - {\frac{1}{n}} \Big)^{k-\ell-1} (1-\delta)^{\ell} \\
& \leq \big( 1 - \gamma a + \gamma^2  b \big) (1-\delta)^{k} + \frac{ \gamma^2 c }{n} (1 - \delta)^{k-1} \frac{1}{1 - \frac{1-1/n}{1-\delta}} \\
& = (1-\delta)^{k} \left\{ \big( 1 - \gamma a + \gamma^2  b \big) + \gamma^2 c \frac{1}{1 - n \delta}  \right\} \\
& \overset{(a)}{\approx} (1-\delta)^{k} \left\{ \big( 1 - \gamma a + \gamma^2  b \big) + \gamma^2 c (1 + \delta n )  \right\} \\
& \leq (1-\delta)^{k} \left\{ \big( 1 - \gamma a + \gamma^2  b \big) + \gamma^2 c (1 + n ) \right\}
\end{split}
\eeq
where the approximation holds if $n \delta \ll 1$. 
Lastly, if 
\beq
\gamma \leq \frac{a}{2} (b + c (1+n))^{-1}
\eeq
Then $h_{k+1} \leq (1-\delta)^{k+1} h_0$ with $\delta \leq \gamma a - \gamma^2 (b + c(1+n)) = {\cal O}(1/n)$.
\end{proof}

\end{document}